\documentclass[twoside]{article}

\usepackage[accepted]{aistats2022}
\usepackage[round]{natbib}

\setcitestyle{authoryear,round,citesep={;},aysep={,},yysep={;}}
\renewcommand{\cite}{\citep}

\usepackage{amsthm,amsmath,bbm,amsfonts,amssymb}
\usepackage{dsfont} 
\usepackage{mathtools}
\usepackage{commath}
\mathtoolsset{showonlyrefs=false}

\usepackage{xspace}
\usepackage[T1]{fontenc}
\usepackage[utf8]{inputenc}
\usepackage[usenames,dvipsnames]{xcolor} 

\usepackage{hyperref,url}
\usepackage{cleveref} 

\usepackage{wrapfig}
\usepackage[export]{adjustbox}
\usepackage{algorithm,algorithmic}
\usepackage{algorithmic}
\usepackage{graphicx,tabularx}
\usepackage{multirow,hhline}
\usepackage{booktabs}

\usepackage{capt-of}

\allowdisplaybreaks 

\usepackage{pbox} 

\usepackage[export]{adjustbox}

\usepackage{pifont}
\newcommand{\cmark}{\ding{51}}%

\newcommand{\xmark}{\ding{55}}%
\usepackage{framed}
\usepackage[most]{tcolorbox}
\definecolor{kjgray}{rgb}{.7,.7,.7}

\usepackage{mdframed}
\usepackage{lipsum}
\definecolor{kjgray}{rgb}{.7,.7,.7}
\makeatletter



\makeatletter
\renewcommand{\paragraph}{%
  \@startsection{paragraph}{4}%
  {\z@}{0.50ex \@plus 1ex \@minus .2ex}{-1em}%
  {\normalfont\normalsize\bfseries}%
}
\makeatother

\usepackage[customcolors,shade]{hf-tikz} 

\def\ddefloop#1{\ifx\ddefloop#1\else\ddef{#1}\expandafter\ddefloop\fi}
\def\ddef#1{\expandafter\def\csname c#1\endcsname{\ensuremath{\mathcal{#1}}}}
\ddefloop ABCDEFGHIJKLMNOPQRSTUVWXYZ\ddefloop
\def\ddef#1{\expandafter\def\csname b#1\endcsname{\ensuremath{{\boldsymbol{#1}}}}}
\ddefloop ABCDEFGHIJKLMNOPQRSUVWXYZabcdefghijklmnopqrtsuvwxyz\ddefloop
\def\ddef#1{\expandafter\def\csname h#1\endcsname{\ensuremath{\hat{#1}}}}
\ddefloop ABCDEFGHIJKLMNOPQRSTUVWXYZabcdefghijklmnopqrsuvwxyz\ddefloop 
\def\ddef#1{\expandafter\def\csname hc#1\endcsname{\ensuremath{\widehat{\mathcal{#1}}}}}
\ddefloop ABCDEFGHIJKLMNOPQRSTUVWXYZ\ddefloop
\def\ddef#1{\expandafter\def\csname t#1\endcsname{\ensuremath{\tilde{#1}}}}
\ddefloop ABCDEFGHIJKLMNOPQRSTUVWXYZabcdefghijklmnopqrtsuvwxyz\ddefloop
\def\ddef#1{\expandafter\def\csname wt#1\endcsname{\ensuremath{\widetilde{#1}}}}
\ddefloop ABCDEFGHIJKLMNOPQRSTUVWXYZabcdefghijklmnopqrtsuvwxyz\ddefloop
\def\ddef#1{\expandafter\def\csname r#1\endcsname{\ensuremath{\mathring{#1}}}}
\ddefloop ABCDEFGHIJKLMNOPQRSTUVWXYZabcdefghijklmnopqrtsuvwxyz\ddefloop
\def\ddef#1{\expandafter\def\csname bar#1\endcsname{\ensuremath{\bar{#1}}}}
\ddefloop ABCDEFGHIJKLMNOPQRSTUVWXYZabcdefghijklmnopqrtsuvwxyz\ddefloop
\def\ddef#1{\expandafter\def\csname wbar#1\endcsname{\ensuremath{\overline{#1}}}}
\ddefloop ABCDEFGHIJKLMNOPQRSTUVWXYZabcdefghijklmnopqrtsuvwxyz\ddefloop
\def\ddef#1{\expandafter\def\csname tc#1\endcsname{\ensuremath{\widetilde{\mathcal{#1}}}}}
\ddefloop ABCDEFGHIJKLMNOPQRSTUVWXYZ\ddefloop

\def\hmu{{\ensuremath{\hat{\mu}}} }

\DeclareMathOperator{\EE}{\mathbb{E}}

\DeclareMathOperator{\PP}{\mathbb{P}}

\DeclareMathOperator{\one}{\mathds{1}\hspace{-.18em}}
\DeclareMathOperator{\Reg}{{\mathsf{Reg}}}

\def\RR{{\mathbb{R}}}

\newcommand*\diff{\mathop{}\!\mathrm{d}}

\newcommand{\sr}[2]{ {\stackrel{#1}{#2}} }

\newcommand{\fr}[2]{ { \frac{#1}{#2} }}

\def\lt{\left}
\def\rt{\right}

\def\rarrow{\ensuremath{\rightarrow}\xspace} 
\def\sig{\ensuremath{\sigma}\xspace}

\def\eps{\ensuremath{\varepsilon}\xspace}

\def\om{{\ensuremath{\omega}\xspace} }

\newcommand{\what}[1]{ {\ensuremath{\widehat{#1}}} }

\def\hDelta{\ensuremath{\what{\Delta}}\xspace}

\def\lcl{\lceil}  
\def\rcl{\rceil}  

\makeatletter
\newcommand{\vast}{\bBigg@{3}}
\newcommand{\Vast}{\bBigg@{4}}
\makeatother

\def\lam{\ensuremath{\lambda}}

\def\lam{{\ensuremath{\lambda}\xspace} }

\def\cF{{\ensuremath{\mathcal{F}}}}

\def\cd{\cdot}

\newcounter{textcnt}

\newcommand\addtext[1]{%
  \stepcounter{textcnt}%
  \csgdef{text\thetextcnt}{#1}}

\newcounter{colnum}

\newcounter{Jidx}
\newcommand\dsymhelper[2]{
  \addtext{\hyperlink{#1}{#2}}%
  \blue{\hypertarget{#1}{#2}}%
}
\newcommand\dsym[1]{
  \stepcounter{Jidx}
  \xdef\tmpname{Jsym.\theJidx}
  \expandafter\dsymhelper\expandafter{\tmpname}{#1}
} 

\usepackage[utf8]{inputenc} 
\usepackage[T1]{fontenc}    
\usepackage{hyperref}       
\usepackage{url}            
\usepackage{booktabs}       
\usepackage{amsfonts}       
\usepackage{nicefrac}       
\usepackage{microtype}      
\usepackage[dvipsnames]{xcolor}         

\newcommand\myshade{85}
\colorlet{mylinkcolor}{blue}
\ifdefined\nohyperref\else\ifdefined\hypersetup
\hypersetup{ %
  pdftitle={},
  pdfauthor={},
  pdfsubject={},
  pdfkeywords={},
  pdfborder=0 0 0,
  pdfpagemode=UseNone,
  colorlinks=true,
  linkcolor=mylinkcolor!\myshade!black,
  citecolor=mylinkcolor!\myshade!black,
  filecolor=mylinkcolor!\myshade!black,
  urlcolor=mylinkcolor!\myshade!black,
  pdfview=FitH%
}

\usepackage{anyfontsize}
\usepackage{MnSymbol} 

\usepackage[shortlabels]{enumitem}
\setlist[itemize]{topsep=.5pt,itemsep=0pt,parsep=2pt}
\setlist[enumerate]{topsep=.5pt,itemsep=0pt,parsep=2pt}

\newtheoremstyle{plain}
{3pt}   
{-3pt}   
{\itshape}  
{0pt}       
{\bfseries} 
{.}         
{5pt plus 1pt minus 1pt} 
{}          

\newtheoremstyle{plain2}
{3pt}   
{-3pt}   
{}  
{0pt}       
{\bfseries} 
{.}         
{5pt plus 1pt minus 1pt} 
{}          

\theoremstyle{plain}
\newtheorem{thm}{Theorem}

\newtheorem{cor}[thm]{Corollary}

\theoremstyle{plain2}

\definecolor{kjcolor}{RGB}{46,139,87}

\newcommand{\kjold}[1]{{\textcolor{blue}{[KJ: #1]}}}
\newcommand{\blue}[1]{{\color[rgb]{.3,.5,1}#1}}
\renewcommand{\blue}[1]{#1}

\newcommand{\added}[1]{{\color{Orange}#1}}

\renewcommand{\added}[1]{#1}
\renewcommand{\kjold}[1]{}
\renewcommand{\blue}[1]{#1}

\newcommand{\removeforfinal}[1]{#1}
\renewcommand{\removeforfinal}[1]{}

\def\MSP{MS$^+$\xspace}

\usepackage[toc,page,header]{appendix}
\usepackage{minitoc}
\usepackage{silence}
\renewcommand \thepart{}
\renewcommand \partname{}

\usepackage{bibunits}      

\begin{document}

\setlength{\abovedisplayskip}{4pt}
\setlength{\belowdisplayskip}{4pt}
\setlength{\abovedisplayshortskip}{4pt}
\setlength{\belowdisplayshortskip}{4pt}
  
\doparttoc 
\faketableofcontents 


%

%

\twocolumn[

\aistatstitle{Maillard Sampling: Boltzmann Exploration Done Optimally}

\aistatsauthor{ Jie Bian \And Kwang-Sung Jun }

\aistatsaddress{ University of Arizona \And  University of Arizona} ]

\begin{abstract}
  The PhD thesis of Maillard (2013) presents a rather obscure algorithm for the $K$-armed bandit problem.
  This less-known algorithm, which we call Maillard sampling (MS), computes the probability of choosing each arm in a \textit{closed form}, which is not true for Thompson sampling, a widely-adopted bandit algorithm in the industry.
  This means that the bandit-logged data from running MS can be readily used for counterfactual evaluation, unlike Thompson sampling.
  Motivated by such merit, we revisit MS and perform an improved analysis to show that it achieves both the asymptotical optimality and $\sqrt{KT\log{T}}$ minimax regret bound where $T$ is the time horizon, which matches the known bounds for asymptotically optimal UCB. 
  We then propose a variant of MS called MS$^+$ that improves its minimax bound to $\sqrt{KT\log{K}}$.
  MS$^+$ can also be tuned to be aggressive (i.e., less exploration) without losing the asymptotic optimality, a unique feature unavailable from existing bandit algorithms.
  Our numerical evaluation shows the effectiveness of MS$^+$.
\end{abstract}

\section{INTRODUCTION}

The $K$-armed bandit problem~\cite{thompson33onthelikelihood,lattimore20bandit} poses a unique challenge of balancing between exploration and exploitation in sequential decision-making tasks.
Researchers have studied this problem and propose algorithms extensively since its widely adopted commercial applications including online news recommendations~\cite{li10acontextual}.
In this problem, the learner is given a set of $K$ arms where each arm $i$ has an associated reward distribution $\nu_i$ with an unknown mean $\mu_i \in\RR$.
At each time $t$, the learner chooses an arm $I_t \in [K]:=\{1,\ldots,K\}$ and then receives a reward $y_t \sim \nu_i$ associated with the chosen arm $I_t$.
In this paper, we focus on the assumption that $y_t - \mu_{I_t}$ is $\sig^2$-sub-Gaussian, i.e.,  $\EE[\exp(\lam \eta_t)] \le \exp(\lam^2 \sig^2/2)$, and that the algorithm knows $\sig^2$, which is standard.\footnote{
  In many cases, practitioners have access to a reasonable value of $\sig^2$.
  For example, clicks ($\{0,1\}$) or scaled 5-star ratings ($\{1,2,3,4,5\}$) are bounded, and if a reward is bounded by $[a,b]$ then it is $\fr{(a-b)^2}{4}$-sub-Gaussian.}
The standard performance measure of a bandit algorithm is cumulative (pseudo-)regret, or simply \textit{regret}, defined as follows:
\begin{align}\label{eq:regret}
    \Reg_T = \sum_t^T \max_{a\in[K]} \mu_a - \mu_{I_t}
\end{align}
where $T$ is the time horizon.

Among many algorithms with regret guarantees, randomized algorithms such as Thompson sampling have attracted significant attention due to their high performance and the fact that the actions are chosen randomly~\cite{chapelle11anempirical}.
The latter property is useful for off-policy evaluation where one has access to a set of chosen arms and their rewards as a result of deploying a bandit algorithm $A$ and desires to estimate ``How much reward would I have collected, had I used a different algorithm $B$?''~\cite{precup00eligibility,li11unbiased,li15toward}.
Solutions for such a question are attractive since they enable the evaluation of alternative algorithms without deploying thus are an active area of research.
However, all the existing methods with good regret guarantees require that the probability value with which the pulled arm was chosen is available precisely; e.g., inverse propensity scoring~\cite{horvitz52generalization} and doubly robust estimators~\cite{robins95semiparametric}.
Thompson sampling, however, does not maintain explicit probability values, and there is no closed-form expression for computing them.
While one can simulate the sampling rule from Thompson sampling and estimate the probability, this is computationally expensive especially for web-scale deployments.

Surprisingly, there exists a novel randomized algorithm in the PhD thesis of~\citet{maillard13apprentissage} that computes the probability of choosing each arm in a closed form, which we call \textit{Maillard sampling} (MS).
At time step $t$, MS for $\sig^2 = 1$ chooses arm $a$ with probability
\begin{align*}
  p_{t,a} \propto \exp\del{-\frac{N_{t-1,a}}{2} \hDelta_{t-1,a}^2}
\end{align*}
where $N_{t-1,i}$ is the number of times arm $i$ has been pulled up to (and including) time step $t-1$, $\hDelta_{t-1,i}:= \max_j \hmu_{t-1,j} - \hmu_{t-1,i} $ is the empirical gap, and $\hmu_{t-1,i}$ is the empirical mean up to time step $t-1$.
The author shows a regret bound of MS that is asymptotically optimal, although 
the bound does not satisfy the sub-UCB-ness and implies a minimax regret of $O(\sqrt{K}T^{3/4})$ that is substantially higher than $\sqrt{KT\log{T}}$ achieved by existing algorithms like UCB1~\cite{auer02finite} (see Section~\ref{sec:prelim} for the definitions of the optimalities).
We provide a detailed comparison of MS and other algorithms in Section~\ref{sec:related}.

MS is a lesser-known type of bandit algorithm that can be viewed as a correction to the popular Boltzmann exploration (BE)~\cite{sutton90integrated,kaelbling96reinforcement}, which chooses arm $i$ with $p_{t,i} \propto \exp(\eta_t \hmu_{t-1,i})$ for some step size $\eta_t$.
That is, since BE can be rewritten as $p_{t,i} \propto \exp(-\eta_t \hDelta_{t-1,a})$, one can obtain MS by squaring the empirical gap and taking arm-specific step size of $N_{t-1,i}/2$ in place of $\eta_t$.
For practitioners, we recommend that they use Eq.~\eqref{eq:practical} as a replacement to BE.

Motivated by the simplicity of MS and its friendliness to the off-policy evaluation, we revisit MS and make two main contributions.
First, we streamline the proof of its theoretical guarantee\footnote{
  In fact, we were not able to verify the correctness of the original proof. The original author has confirmed that the proof is not clear.
} and show that MS achieves not only asymptotically optimality but also the sub-UCB-ness and a minimax regret bound of $O(\sqrt{KT\log(T)})$, matching the guarantees of asymptotically optimal UCB~\citet[Section 8]{lattimore20bandit}. 
Second, we propose a new algorithm called Maillard Sampling$^+$ (\MSP).
Besides all the optimalities that MS possesses, \MSP further enjoys a minimax regret bound of $O(\sqrt{KT\log(K)})$.
We present the algorithms and the regret bounds of MS and \MSP in Section~\ref{sec:ms} and~\ref{sec:msp} respectively.

\MSP further possesses a unique, practically relevant property.
In the design of \MSP, we allow an exploitation parameter that we call the `booster', which encourages exploitation without breaking theoretical guarantees. 
Increasing the booster parameter only affects a lower-order term in the regret bound.
This means that even for a very aggressive booster, which is likely to incur linear regret for some time in the beginning, the algorithm has the ability to recover from it in a finite time. 
This is in stark contrast to the common heuristic of reducing the confidence width of UCB~\cite{zhang16online} or posterior variance of Thompson sampling~\cite{chapelle11anempirical}, which breaks the regret guarantee and risks suffering a linear regret from which one may never recover.
The practical implication is that, given domain knowledge on the noise level, one can tune the booster parameter to perform well without such a risk.

In section~\ref{sec:expr}, we perform an empirical evaluation of \MSP, which achieves either the best or on par regret among the anytime algorithms. 
Our evaluation also shows the effect of the booster parameter, which indeed allows us to reduce the expected regret without risking a linear regret at the price of increasing variance.
Finally, we conclude our paper in Section~\ref{sec:future} with future research directions enabled by our paper. 

\definecolor{mygrn}{rgb}{0,.8,0}
\definecolor{myred}{rgb}{.8,0,0}

\newcommand{\mycm}{\textcolor{mygrn}{\cmark}}
\newcommand{\myxm}{\textcolor{myred}{\xmark}}

\textfloatsep=1.2em

\newcolumntype{P}[1]{>{\centering\arraybackslash}p{#1}}
\begin{table*}
  {\centering
    {\footnotesize
    \begin{tabular}{ccP{0.10\linewidth}P{0.06\linewidth}cP{0.10\linewidth}}
      \hline
      & Sub-UCB & Asympt. optimality & Minimax ratio         & Anytime & Closed-form probability \\\hline
      UCB1~\cite{auer02finite}      &    \mycm    &           \myxm            &  $\sqrt{\log{T}}$     &    \mycm   & N/A\\
      AOUCB~\cite[Sec. 8]{lattimore20bandit}      &    \mycm    &           \mycm            &  $\sqrt{\log{T}}$     &    \mycm  & N/A\\
      TS-SG~\cite[Alg. 2]{agrawal17nearoptimal}    &    \mycm    &           \myxm            &  $\sqrt{\log{K}}$     &    \mycm & \myxm\\
      MOSS~\cite[Sec. 9]{lattimore20bandit}      &    \myxm    &           \myxm            &  $1$                  &    \myxm  & N/A\\
      UCB$^+$~\cite[Alg. 2]{lattimore15optimally}      &    \mycm    &           \myxm*            &  $\sqrt{\log{K}}$     &    \myxm & N/A\\ 
      OCUCB~\cite[Alg. 1]{lattimore15optimally}      &    \mycm    &           \myxm*            &  1      &    \myxm & N/A\\ \hline
      MS (\textbf{this work}; Algorithm~\ref{alg:cap})       &    \mycm    &           \mycm            &  $\sqrt{\log{T}}$     &    \mycm  & \mycm\\
      MS$^+$ (\textbf{this work}; Algorithm~\ref{alg:ims})      &    \mycm    &           \mycm            &  $\sqrt{\log{K}}$     &    \mycm & \mycm \\\hline
    \end{tabular}
  
  }
    \caption{
      Comparison of bandit algorithms for sub-Gaussian rewards.
      Our algorithms enjoy the best guarantee except for the minimax ratio.
      See Section~\ref{sec:related} for details.
      The mark \myxm* means that the optimality is not explicitly reported but might actually be achieved.
    }
    \label{tab:comparison}
  }
\end{table*}

\section{PRELIMINARIES}
\label{sec:prelim}

\textbf{Notations.}
Without loss of generality, we assume $\blue{\mu_1\ge \mu_2 \ge \cdots \ge \mu_K}$.
Define the gap $\blue{\Delta_a} := \mu_1 - \mu_a$, $\forall i\in[K] := \cbr{1,\ldots,K}$.

\textbf{Regret optimality.}
For the expected regret $\EE \Reg_T$, there are multiple optimality criteria.
An algorithm is called asymptotically optimal if it satisfies the following \cite{robbins85asymptotically,burnetas96optimal}:
\begin{align}\label{eq:asymp}
  \lim \sup_{T\rarrow\infty} \fr{\EE\Reg_T}{\log(T)}  = \sum_{i: \Delta_i > 0} \fr{2\sigma^2}{\Delta_i} ~.
\end{align}

The minimax optimal regret is the smallest achievable worst-case regret with respect to the number of arms $K$ and the time horizon $T$.
The minimax optimal regret of the bandit problem is $\Theta(\sqrt{KT})$~\cite{audibert09minimax}, meaning that there exists an algorithm with the minimax regret of $O(\sqrt{KT})$ and a problem instance where any algorithm must suffer regret of $\Omega(\sqrt{KT})$.
We say that an algorithm has a minimax ratio of $f(K,T)$ if it has a minimax regret bound of $O(\sqrt{KT}f(K,T))$.
Finally, an algorithm is said to be sub-UCB if $\EE \Reg_T \le C_1 \sum_{i\in[K]} \Delta_i + C_2 \sum_{i:\Delta_i>0} \frac{\sigma^2}{\Delta_i} \log(T)$ for absolute constants $C_1,C_2 >0$~\cite{lattimore20bandit}.
While sub-UCB is satisfied by most existing algorithms including UCB1~\cite{auer02finite}, some algorithms like MOSS~\cite{audibert09minimax} are so aggressive on certain instances that they are not sub-UCB; see \citet{lattimore20bandit} for details.

An algorithm $\cA$ is said to be anytime if it does not take in the time horizon $T$ whereas $\cA$ is fixed budget if it requires $T$ as input and enjoys a regret bound only for the time horizon $T$.

\section{RELATED WORK}
\label{sec:related}

We summarize existing $K$-armed bandit algorithms for sub-Gaussian rewards in Table~\ref{tab:comparison}.
In this table, TS-SG stands for Thompson sampling for sub-Gaussian rewards, which is originally designed for bounded rewards ($y_t \in [0,1]$) yet trivially enjoys the same guarantee for $(1/2)^2$-sub-Gaussian rewards. 
We emphasize that TS-SG is different from TS for Gaussian rewards~\cite{korda13thompson} as TS-SG inflates the posterior variance by a factor of $4$, which is why it does not have the asymptotic optimality.
Note that there are much more algorithms if we restrict to Gaussian rewards.
We refer to~\citet[Table 2]{lattimore20bandit} for a more comprehensive summary.

Note that we only report results available in the literature in Table~\ref{tab:comparison}.
For example, while MOSS and UCB$^+$ are fixed-budget algorithms, one should be able to modify these algorithms and prove \textit{anytime} guarantees using existing techniques.
Such a modification comes with a cost of inflating the width of the confidence bound, which likely  degrades the practical performance.
Similarly, one should be able to modify AdaUCB~\cite{lattimore18refining}, which achieves all the possible optimality criteria in the literature for \textit{Gaussian} rewards, to allow \textit{sub-Gaussian} rewards. 
However, we emphasize that the focus of this paper is the (re-)introduction of a new type of algorithm with desirable properties and good regret guarantees, leaving further improvements in regret bounds for future work.

Finally, MS is closely related to many existing algorithms including Thompson sampling, we review in detail in Appendix~\ref{sec:relatedalg}.

After the submission of the paper, the authors of~\citet{honda15non} have informed us that MS is in fact an instantiation of MED~\cite{honda11asymptotically} for Gaussian rewards and that at that time (`pre-Thompson-sampling' era) randomized algorithms were not preferred by reviewers, which necessitated deterministic versions of MED such as DMED~\cite{honda10asymptotically} and IMED \cite{honda15non}.
However, the analysis of MED makes a lot of assumptions and is inherently an asymptotic argument.
In contrast, our analysis of MS is finite time, and shows that it achieves a near-optimal minimax regret for the first time.

\section{MAILLARD SAMPLING}
\label{sec:ms}

\begin{algorithm}[t]
\caption{Maillard Sampling}\label{alg:cap}
\begin{algorithmic}
   \STATE \textbf{Input:} $K \ge 2$, $\sigma^2 > 0$.
   \FOR{$t=1,2,...$}
   \IF{$t\le K$}
   \STATE Pull the arm $I_t=t$ and observe reward $y_t$.
   \ELSE
   \STATE $p_{t,a} \propto \exp(-\fr{1}{2\sigma^2}N_{t-1,a}\hDelta^2_{t-1,a})$
   \STATE Pull the arm $I_t \sim p_t$ and observe reward $y_t$.
   \ENDIF
   \ENDFOR
\end{algorithmic}
\end{algorithm}

We now formally introduce Maillard sampling (MS), which originally appeared in~\citet[Figure 1.8]{maillard13apprentissage}.
Algorithm~\ref{alg:cap} describes the pseudocode of MS with an explicit dependence on the sub-Gaussian parameter $\sigma^2$.
MS has a very simple closed-form computation of sampling probability that directly uses the estimated gap $\hDelta_{t,a}$.
Interestingly, the sampling probability $p_{t,a}$ is proportional to the sub-Gaussian tail at deviation $\hDelta_{t-1,a}$.
Such a form of probability has a close connection to Thompson sampling as we discuss in Section~\ref{sec:relatedalg}.

\begin{thm}\label{thm:ms0}\cite[Theorem 1.10]{maillard13apprentissage}
Maillard sampling (Algorithm~\ref{alg:cap}) with $\sigma^2 = 1/4$ satisfies that, $\forall T\ge1$ and $c>0$, 
\begin{align*}
    \EE\Reg_T
    \le \sum_{a:\Delta_a>0} \del{\frac{(1+c)^2}{2\Delta_a}\log(T) + O\del{\frac{(1+c)^4}{c^4\Delta_a^3}K}}~.
\end{align*}
\end{thm}
Theorem~\ref{thm:ms0} implies that MS is asymptotically optimal by choosing $c = 1/o(\log^{1/4}(T))$.
Specifically, such a value of $c$ makes $(1+c)^2$ asymptotically 1 and $\frac{(1+c)^4}{c^4}$ asymptotically $o(\log(T))$.
However, Theorem~\ref{thm:ms0} does not imply that MS is sub-UCB due to the second term with a larger dependence on $1/\Delta_a$.
Furthermore, using the standard technique such as~\citet[Theorem 7.2]{lattimore20bandit}, one can show that Theorem~\ref{thm:ms0} implies a minimax regret bound of $\EE\Reg_T = O(\sqrt{K} T^{3/4})$.

Altogether, Theorem~\ref{thm:ms0} implies that MS has a worse regret upper bound than basic UCB algorithms such as AOUCB (asymptotically optimal UCB)~\cite{cappe13kullback,lattimore20bandit}, which is not only asymptotically optimal but also sub-UCB along with minimax regret of $O(\sqrt{KT\log(T)})$.
Is this an intrinsic limitation of MS or merely a consequence of a loose analysis?

Our main result below shows that it is indeed a consequence of a loose analysis and that MS enjoys the same regret guarantees as AOUCB.
\begin{thm}
\label{thm:ms}
MS satisfies that, $\forall T\ge1$ and $c>0$, 
\begin{align*}
&  \EE\Reg_T
\le \sum_{a\in[K]: \Delta_a > 0}^{}\Bigg(\fr{2\sigma^2(1+c)^2 \ln(\fr{T\Delta_a^2}{\sigma^2} )}{\Delta_a} \\
&\qquad\quad
+ O\del{\Delta_a \vee \del{\fr{(1+c)^2\sigma^2}{c^2\Delta_a}\ln\fr{(1+c)^2\sigma^2}{c^2\Delta_a^2}}}\Bigg)~.
\end{align*}
\end{thm}
Theorem~\ref{thm:ms} implies that MS is sub-UCB (note $\ln(T\Delta_a^2) + \ln(\frac{1}{\Delta_a^2}) = \ln(T)$).
Setting $c = 1/o(\ln^{1/2}(T))$, we achieve the asymptotic optimality in the same way described above.
Finally, we show the minimax reget bound in the following corollary as a consequence of Theorem~\ref{thm:ms}.
Hereafter, all the proofs are in the appendix unless noted otherwise.
\begin{cor}
\label{thm1_cor}
MS satisfies that that, $\forall T\ge1$,
\begin{align*}
    \EE\Reg_T \le
    O(\sigma\sqrt{KT\ln(T)})~.
\end{align*}
\end{cor}

\subsection{Proof of Theorem~\ref{thm:ms}}

Let $T \ge K$.
Let $\blue{N_{t,a}}$ be the number of times arm $a$ has been pulled up to (and including) time step t.
Using $\EE \Reg_T = \sum_{a:\Delta>0} \Delta_a \EE[N_{t,a}]$, it suffices to bound $\EE[N_{T,a}]$.
Fix $a \in [K]$ such that $\Delta_a >0$ and define
\begin{align*}
  \blue{u} := \lt\lcl \fr{2\sigma^2(1+c)^2 \ln(\fr{T\Delta_a^2}{2\sigma^2})}{\Delta_a^2} \rt\rcl~.
\end{align*}
\kjold{Let's add $\vee 1$ to the above. }
Then,
\begin{align*}
  \EE[N_{T,a}] 
    &\le u + \EE\sbr{\sum_{t=u+1}^T \one\cbr{ I_t = a \text{ and } N_{t-1,a} > u}}
  \\&=   u + \EE\sbr{\sum_{t=u}^{T-1} \one\cbr{ I_{t+1} = a \text{ and } N_{t,a} > u}}~.
\end{align*}
Let $\blue{\hmu_{t,\max}} = \max_{a} \hmu_{t,a}$.
Let $\eps>0$ be a variable to be tuned later. 
Recall that $\blue{\Delta_a}= \mu_1-\mu_a$ and $\blue{\hDelta_{t,a}}=\hmu_{t,\max}-\hmu_{t,a}$.
We further split the second term as follows:
\begin{align*}
&\EE\sbr{\sum_{t=u}^{T-1} \one\cbr{ I_{t+1} = a \text{ and } N_{t,a} > u}}
\\&=\EE \sum_{t=u}^{T-1} \one\cbr[2]{I_{t+1} = a, N_{t,a} > u, \hmu_{t,\max} \ge \mu_1 - \eps,~ \hDelta_{t,a} > \fr{\Delta_a}{1+c}}
\\&+\EE \sum_{t=u}^{T-1} \one\cbr[2]{I_{t+1} = a, N_{t,a} > u, \hmu_{t,\max} \ge \mu_1 - \eps,~ \hDelta_{t,a} \le \fr{\Delta_a}{1+c}}
\\&+
\EE \sum_{t=u}^{T-1} \one\cbr[2]{I_{t+1} = a, N_{t,a} > u, \hmu_{t,\max} < \mu_1 - \eps}~.
\end{align*}
We call the three terms above (F1), (F2), and (F3) respectively.

The term (F1) is a desirable event where both the empirically best mean is close to $\mu_1$ and the gap estimation of arm $a$ is good enough.
Then,
\begin{align*}
  &\text{(F1)}
  \\&\le  \sum_{t=u}^{T-1} \PP\del{  A_{t+1} = a \mid N_{t,a} > u,~ \hDelta_{t,a} > \fr{\Delta_a}{1+c}} 
      \\&\qquad \times\PP\del{ N_{t,a} > u,~ \hDelta_{t,a} > \fr{\Delta_a}{1+c}} 
  \\&\le  \sum_{t=u}^{T-1} \exp\del{-\fr{1}{2\sigma^2}u\del{\fr{\Delta_a}{1+c}}^2} \cd 1
      \le \fr{2\sigma^2}{\Delta_a^2}~.
\end{align*}

For the term (F2), it is easy to see that
\begin{align*}
  \cbr{\hmu_{t,\max} \ge \mu_1 - \eps} \cap \cbr{ \hDelta_{t,a} \le \fr{1}{1+c} \Delta_a} \\\subseteq \cbr{\hmu_{t,a} - \mu_a \ge \fr{c}{1+c}\Delta_a - \eps}~.
\end{align*}
Let  $\blue{\tau_{a}(k) }$ be the time step $t$ such that the arm a was pulled at $t$ and the number of arm pulls of arm a becomes $k$ at the end of the time step $t$; i.e., $\tau_a(k) = \min\cbr{t\ge1: N_{t,a} = k}$. 
We use the shortcut $\blue{\hmu_{a,(k)}} := \hmu_{\tau_a(k),a}$ to denote the sample mean of arm $a$ after pulling it for the $k$-th time, . 
Thus,
\begin{align*}
    &\EE \sum_{t=u}^{T-1} \one\cbr{I_{t+1} = a,~ N_{t,a} > u,~ \hmu_{t,\max} \ge \mu_1 - \eps,~ \hDelta_{t,a} \le \fr{\Delta_a}{1+c}}
  \\&\le \sum_{t=u}^{T-1} \PP\del{I_{t+1} = a,~ \hmu_{t,a} - \mu_a \ge \fr{c\Delta_a}{1+c} - \eps}
  \\&\le \EE \sbr{\sum_{k=1}^{\infty} \sum_{t=\tau_1(k)}^{\tau_1(k+1)-1} \one \cbr{I_{t+1} = a} \cd \one \cbr{ \hmu_{t,a} - \mu_a \ge \fr{c\Delta_a}{1+c} - \eps}}
  \\&= \EE \sbr{\sum_{k=1}^{\infty}  \one \cbr{ \hmu_{a,(k)} - \mu_a \ge \fr{c\Delta_a}{1+c} - \eps} \sum_{t=\tau_1(k)}^{\tau_1(k+1)-1}\one \cbr{I_{t+1} = a} }
  \\&= \EE \sbr{\sum_{k=1}^{\infty}  \one \cbr{ \hmu_{a,(k)} - \mu_a \ge \fr{c\Delta_a}{1+c} - \eps}  }
  \\&=\sum_{k\ge1} \PP\del{ \hmu_{a,(k)} - \mu_a \ge \fr{c\Delta_a}{1+c} - \eps} \tag{$ u\ge k$}
  \\&\le \sum_{k\ge1} \exp\del{-\fr{k}{2\sigma^2}\del{\fr{c\Delta_a}{1+c} - \eps}^2}
   \\&= \fr{1}{\exp\del{\fr{1}{2\sigma^2}\del{\fr{c\Delta_a}{1+c} - \eps}^2} - 1}
      \le \fr{2\sigma^2}{\del{\fr{c\Delta_a}{1+c} - \eps}^2}~.
\end{align*}

The term (F3) deals with the case where the best empirical mean is much lower than the true best mean $\mu^*$.
This part of the proof is the most nontrivial part where we essentially show that that exists a sufficient amount of probability assigned to the true best arm so that eventually we can recover from the case where the empirical mean of the best arm is actually not empirically the best.
Let $\blue{\cF_t}$ be the $\sigma$-algebra generated by $(A_1,R_1,\ldots,A_{t},R_{t})$.
We define the shortcuts $\blue{\PP_t(\cE)} := \PP(\cE \mid \cF_t)$ and $\blue{\EE_t[X]} := \EE[X\mid \cF_t]$. 

We first relate the probability of choosing the arm $a$ with that of the true best arm.
If $\hmu_{t,\max} < \mu_1 - \eps$, then
\begin{align}\label{eq:relate}
  \begin{aligned}
      &\PP_t(I_{t+1}=a) 
    \\\quad&= \fr{\exp(-\fr{1}{2\sigma^2}N_{t,a} \cd \hDelta_{t,a}^2)}{\exp(-\fr{1}{2\sigma^2}N_{t,1} \cd \hDelta_{t,1}^2)} \PP_t(I_{t+1}=1)
    \\\quad&\le \fr{1}{\exp(-\fr{1}{2\sigma^2}N_{t,1} \cd (\mu_1-\eps-\hmu_{t,1})_+^2)} \PP_t(I_{t+1}=1)~.
  \end{aligned}
\end{align}
Let $\blue{\tau_{1}(k) }$ be the time step after which the number of arm pulls of just becomes $k$; i.e., $\tau_1(k) = \min\cbr{t\ge1: N_{t,1} = k}$. 
Let us use the shortcut $\blue{\hmu_{1,(k)}} := \hmu_{\tau_1(k),1}$. 
Define $B_t := \one\{\hmu_{t,\max} < \mu_1 - \eps\}$, which is at most $B_{t,1} = \one\{\hmu_{t,1} < \mu_1 - \eps\}$.
Let $A_{t+1,i} = \one\{I_{t+1}=i\}, \forall i\in[K]$.
With this, we can bound (F3) as follows:
\begin{align*}
  &\EE \sum_{t=u}^{T-1} \one\cbr{I_{t+1} = a,~ N_{t,a} > u,~ \hmu_{t,\max} < \mu_1 - \eps}
  \\&\le \EE\sbr[2]{ \sum_{t=u}^{T-1} \EE_t\sbr{B_tA_{t+1,a} }  }
  \\&=   \EE\sbr[2]{ \sum_{t=u}^{T-1} B_t \EE_t\sbr{A_{t+1,a} }  } 
  \\&\le \EE\sbr[2]{ \sum_{t=u}^{T-1} B_t \exp\del{\fr{N_{t,1}  (\mu_1-\eps-\hmu_{t,1})^2}{2\sigma^2}} \EE_t\sbr{ A_{t+1,1} } } \tag{by \eqref{eq:relate}}
  \\&=   \EE\sbr[2]{ \sum_{t=u}^{T-1} \EE_t \sbr[2]{B_t \exp\del{\fr{N_{t,1}  (\mu_1-\eps-\hmu_{t,1})^2}{2\sigma^2}} A_{t+1,1}  }}
  \\&=   \EE\sbr[2]{ \sum_{t=u}^{T-1} B_t \exp\del{\fr{N_{t,1}  (\mu_1-\eps-\hmu_{t,1})^2}{2\sigma^2}} \cd A_{t+1,1}  }
  \\&\le   \EE\sbr[2]{ \sum_{t=u}^{T-1} B_{t,1}\exp\del{\fr{N_{t,1}  (\mu_1-\eps-\hmu_{t,1})^2}{2\sigma^2}} \cd A_{t+1,1}  }
  \\&=   \EE\Big[ \sum_{k=1}^{\infty} \sum_{t=\tau_1(k)+1}^{\tau_1(k+1)} B_{t,1}\exp\del{\fr{N_{t,1}  (\mu_1-\eps-\hmu_{1,(k)})^2}{2\sigma^2}} 
      A_{t+1,1} \Big].
\end{align*}
Using the trivial fact that $\sum_{t=\tau_1(k)+1}^{\tau_1(k+1)} A_{t+1,1} = 1$ and with a new notation $\barX_k := \mu_1 - \hmu_{1,(k)}$, we have that
\begin{align*}
  \text{(F3)} \le \EE\sbr{ \sum_{k=1}^{\infty}  \one\cbr{ \blue{\bar X_k} > \eps}\cd \exp\del{\fr{1}{2\sigma^2} k \cd (\blue{\bar X_k} - \eps)^2} } ~.
\end{align*}
The evaluate this integral, we perform peeling to upper bound the above by
\begin{align*}
  &\EE\Big[ \sum_{k=1}^{\infty}  \sum_{q=1}^\infty \one\cbr{ \mu_1 - \hmu_{1,(k)} \in \lt\lparen \eps+\fr{q-1}{2}\eps,~ \eps + \fr{q}{2} \eps \rt\rbrack  } 
        \\&\qquad \times \exp\del{\fr{1}{2\sigma^2} k \cd \del{\fr{q\eps}{2}}^2} \Big]
  \\&=  \sum_{k=1}^{\infty}  \sum_{q=1}^\infty \PP\del{ \mu_1 - \hmu_{1,(k)} \in \lt\lparen \eps+\fr{q-1}{2}\eps,~ \eps + \fr{q}{2} \eps \rt\rbrack  } 
        \\&\qquad\times \exp\del{\fr{1}{2\sigma^2} k \cd \del{\fr{q\eps}{2}}^2} 
  \\&\le  \sum_{k=1}^{\infty}  \sum_{q=1}^\infty \exp\del{-\fr{1}{2\sigma^2}k \del{\fr{q+1}{2} \eps}^2} \cd \exp\del{\fr{1}{2\sigma^2} k \cd \del{\fr{q\eps}{2}}^2}
  \\&= \sum_{k=1}^{\infty}  \sum_{q=1}^\infty \exp\del{-\fr{1}{4\sigma^2}k\eps^2(q + \fr12)}
  \\&= \sum_{q=1}^\infty   \fr{1}{\exp( \fr{1}{4\sigma^2}(q+\fr12 )\eps^2) - 1}
  \\&\sr{(a)}{\le} \fr{4\sigma^2}{\eps^2} \sum_{q=1}^{\bar q - 1}  \frac{1}{q+\fr12} + \sum_{q=\bar q}^{\infty} \fr{1}{\exp( \fr{1}{4\sigma^2}(q+\fr12 )\eps^2) - 1}
\end{align*}
where $(a)$ introduces a free variable $\bar q \in \{1,2,\ldots\}$ and uses $e^z \ge z + 1$.
Note that as long as $\bar q = \Theta(1)$, this is tight because $X\in(0,1)$ then $1+X \le e^X \le (e-1)X + 1$. 
It remains to bound the above and tune $\barq$.
Leaving the algebra to the appendix, we obtain
\begin{align*}
    \text{(F3)} = O\del{\frac{\sig^2}{\eps^2}(1 + \ln(1 + \frac{\sig^2}{\eps^2}))}~.
\end{align*}
Altogether, choosing $\eps = \fr{c\Delta_a}{2(1+c)}$, we have that
\begin{align*}
 \EE[N_{T,a}] 
 &\le u + \text{(F1)} + \text{(F2)} + \text{(F3)}
 \\&\le u + O\del{\added{\frac{(1+c)^2\sig^2}{c^2\Delta_a^2}(1 + \ln(1 + \frac{(1+c)^2 \sig^2}{c^2\Delta_a^2}))}}
 ~,
\end{align*}
which concludes the proof.

\section{MAILLARD SAMPLING PLUS (\MSP)}
\label{sec:msp}

\begin{algorithm}[t]
\caption{Maillard Sampling$^+$ (MS$^+$)}
\label{alg:ims}
\begin{algorithmic}
   \STATE \textbf{Input:} $B\ge 1, 0\le C\le1, 0<D\le1, \sigma^2>0$
   \FOR{$t=1,2,...$}
   \IF{$t\le K$}
   \STATE Pull the arm $I_t=t$ and observe reward $y_t$.
   \ELSE
   \STATE For all $a\in[K]$ with $\hmu_{t-1,a} = \max_b \hmu_{t-1,b}$,  
          $$\hat p_{t,a} = B \cd (1 + C\ln(1 + \ln(\frac{t}{N_{t-1,a}})))~.$$
   \STATE For other arms,
          \begin{align*}
          \hat p_{t,a} = \exp&(-\fr{1}{2\sigma^2}N_{t-1,a}\hDelta^2_{a,t-1}
          \\&+\ln(1+\fr{D}{2\sigma^2} N_{t-1,a}\hDelta^2_{a,t-1}))~.
          \end{align*}
   \STATE Compute $p_{t,a} = \frac{\hat p_{t,a}}{\sum_b \hat p_{t,b}}$.
   \STATE Pull the arm $I_t \sim p_t$ and observe reward $y_t$.
   \ENDIF
   \ENDFOR
\end{algorithmic}
\end{algorithm}

We now propose an improved version of MS described in Algorithm~\ref{alg:ims}, which we call \MSP.
\MSP differs from MS in two aspects.
First, the probability of being sampled is inflated by a small amount, and the empirical best arms are also inflated but in a slightly different way.
This helps achieving the minimax ratio of $\sqrt{\log(K)}$, which is better than $\sqrt{\log(T)}$ of MS.
Second, we have introduced a number of parameters $B$, $C$, and $D$.
We remark that while the parameters $C$ and $D$ are less important in practice, which we set to $C=D=0.01$ in our experiments later, the parameter $B$, which we call the `booster', affects the performance a lot. 
We recommend $B=4$ to be safe, but depending on the situation one may want to set it to $B=8$ or even a higher value; see our experiments in Section~\ref{sec:expr} for more discussion.

\MSP enjoys the following regret bound.
\begin{thm}
    \label{thm_IMS}
    \MSP (Algorithm~\ref{alg:ims}) satisfies that, for every $T\ge1$ and $c>0$,
    \begin{align*}
      \EE \Reg_T
     \le &\sum_{a:\Delta_a>0} \fr{2\sigma^2(1+c)^2 \ln(\fr{T\Delta_a^2}{2\sigma^2}(1+\ln(1+\fr{T\Delta_a^2}{2\sigma^2})))}{\Delta_a} \\&+ O\del{ B\fr{\sigma^2(1+c)^2\ln(1+\ln(\fr{T\Delta_a^2}{2\sigma^2}))}{c^2\Delta_a} + \Delta_a}
    \end{align*}
    where we omit the dependence on $C$ and $D$ for brevity.
\end{thm}
\begin{proof}
  The key difference from the proof of MS is the term (F3) therein that resulted in $O(\fr{\sig^2}{\Delta^2} \ln(\frac{\sig^2}{\Delta^2}))$.
  Inspecting how we obtain the minimax regret bound, one can see that we would obtain the minimax ratio of $\sqrt{\log(K)}$ if we could bound (F3) by $O(\fr{\sig^2}{\Delta^2})$.
  Indeed, this was the driving force behind the design of \MSP.
  The inflated sampling probability helps us to obtain such a bound.
  We refer to our appendix for the full proof.
\end{proof}
Just like MS, \MSP is sub-UCB, one can obtain an asymptotic optimality by setting $c = 1/o(\ln^{1/2}(T))$.
One can see that the dependence on $B$ is on the lower order term only.
This means that however large $B$ is, for large enough $T$, the leading term will dominate the regret.

The following corollary shows that we can obtain a minimax bound of $\sqrt{KT\log(K)}$, which is an improvement of MS.
\begin{cor}
\label{thm2_cor}
\MSP with absolute constants $B$, $C$, and $D$ satisfies that, $\forall T\ge1$,
\begin{align*}
    \EE \Reg_T\le
    O(\sigma\sqrt{KT\ln(K)})~.
\end{align*}
\end{cor}

We remark that \MSP enjoys the same minimax bound with $C=0$, and the parameter $C$ can be viewed as another way to encourage exploitation that goes away with a large enough $t$, which is has a milder effect compared to $B$.

\section{EXPERIMENTS}
\label{sec:expr}

In this section, we numerically investigate the effects of the parameter $B$ of MSP and evaluate the performance of MSP with existing bandit algorithms for sub-Gaussian rewards.
Throughout, we set $C=D=0.01$ for MSP, which seems to work well overall. 

\textbf{The effect of the booster parameter $B$.}
We consider a Gaussian arm set with variance 1.
There are 10 arms with mean rewards that decrease linearly: $\{0.9,0.8,\ldots,0.0\}$.
We run MS$^+$ 200 times with $B \in \{2,4,8,\ldots,128\}$.
Figure~\ref{fig:expr-tuning} plots the average regret and its standard deviation.
The average regret goes down up to $B=16$ and then starts to increase.
The variance starts to increase from $B=8$ and then soon increases dramatically.
This shows the tradeoff between the expected regret and the variance of the regret.
Note, however, that even with a very aggressive booster $\log_2(B)=7$, the average regret is not dramatically higher than the best and certainly far from suffering a linear regret.
This is not surprising because MSP still enjoys a valid regret bound; in fact, even with $B=128$, MSP is asymptotically optimal. 
Such aggressive exploitation of the empirical best arm without losing theoretical guarantee (and without linear regret) is the unique characteristic of MSP that is not found in other existing algorithms.

\textbf{Methods.}
Since there are so many bandit algorithms for various reward settings, we limit the set of baseline methods to those that have a regret guarantee for sub-Gaussian rewards.
For MSP, we use $B\in\{4,8,16\}$, which we denote by MS$^+$4, MS$^+$8, and MS$^+$16 respectively.
We consider the following baseline methods:
\begin{itemize}
  \item TS: Thompson sampling for Gaussian rewards~\cite{korda13thompson}. 
  This is the only baseline that does not have a regret guarantee for sub-Gaussian rewards, which we include as a comparison point.
  \item TS-SG: Gaussian Thompson sampling for bounded rewards by~\citet{agrawal17nearoptimal} that is (trivially) adapted for sub-Gaussian rewards. The theoretical guarantees for bounded rewards transfer to the sub-Gaussian rewards naturally. The key difference is that this algorithm inflates the posterior variance by a factor of $4$ for technical reasons.
  \item AOUCB: Asymptotically optimal UCB~\cite{lattimore20bandit}.
  \item UCB$^{+}$: Almost optimally confidence UCB  (a.k.a. UCB plus)~\cite{lattimore15optimally}. This is a fixed-budget algorithm (more comments below).
  \item OCUCB: Optimally confidence UCB~\cite{lattimore15optimally}, a fixed-budget algorithm. 
  \item MOSS: MOSS~\cite{audibert09minimax}, a fixed-budget algorithm.
\end{itemize}
Note that the comparison will not be entirely fair since the fixed-budget algorithms use the extra information of the time horizon $T$ to adjust the amount of exploration.
Note that it should be possible to adjust UCB$^{+}$ and MOSS to not require $T$ as input such as~\citet{menard17minimax}, but this will inflate the confidence width, likely increasing regret in practice.
On the other hand, it is not clear to us how to make OCUCB anytime.
All the algorithms were implemented so it takes in the sub-Gaussian parameter $\sigma^2=1/4$.

\textbf{Synthetic bandit problems.}
We test the algorithms above with the bandit problems summarized in Table~\ref{tab:problems}.
For Bernoulli problems, our test algorithms with the sub-Gaussian parameter of $\sigma^2= 1/4$, which is the tightest one without losing their regret guarantees.
The intention of running Bernoulli problems is not to claim that the tested algorithms will excel at them but to observe how they behave under non-Gaussian rewards.
\begin{table}
{\small\centering
  \begin{tabular}{cccc}\hline
    Noise & Gaps & $K$ & $\mu_{1:K}$ \\\hline
    Gaussian & \!Linear\! & \!$\in\cbr{10,100} $  & $\mu_i = 1 - \fr i K$\\
    Gaussian & \!Equal\! & \!$\in\cbr{2,10,100}$  & $\mu_1 = 1, \mu_{2:K} = .5$\\
    Bernoulli & \!Linear\! & \!$\in\cbr{10,100} $   & $\mu_i = 1 - \fr i K$\\
    Bernoulli & \!Equal\!  & \!$\in\cbr{2,10,100}$  & \!$\mu_1 = 0.1, \mu_{2:K} = .05$\\\hline
  \end{tabular}
  }
\caption{Synthetic bandit problems for evaluation.}
\label{tab:problems}
\end{table}

\textbf{Results.}
We run evaluation of the algorithms on every bandit problem in Table~\ref{tab:problems} with $T=20,000$ and 200 trials and report the resulting average regrets and its standard deviations in Figure~\ref{fig:expr-g} and~\ref{fig:expr-b}.
Note that the number of trials is high enough to make the average regret a good estimate of the expected regret except for a few cases with a very large variance.
For Gaussian environments, the overall winner is OCUCB, which is not surprising given that its regret bounds are much better than the rest and that it is a fixed-budget algorithm.
Among anytime algorithms, MS$^+$8, MS$^+$16, and TS are all comparable with respect to the average regret while MS$^+$8 and MS$^+$16 often come with a larger variance.
This is still an encouraging result against TS since TS is optimized for Gaussian rewards; we expect that optimizing MS$^+$ for Gaussian rewards may make MS$^+$ be comparable to TS in both expected regret and the variance.
For the remaining methods, MS$^+$4 has smaller variances than other MS$^+$'s, but its average regrets are mostly behind them, and both TS-SG and AOUCB are much worse than the rest.
We remark that while the variance of MS$^+$16 seems detrimental in Figure~\ref{fig:expr-g}(a), it is still far from suffering a linear regret; the maximum regret of MS$^+$16 over 200 trials was 450.5 while the largest possible regret of any algorithm is $20,000\cd 0.5 = 10,000$.

For Bernoulli instances, we still observe the same trend except for minor differences: (i) the relative performance of TS is not as good as Gaussian environments due to the non-Gaussian nature of the rewards, and (ii) MS$^+$16 consistently has the smallest average regret among anytime algorithms.
The second observation is due to the fact that the actual variances are small.
This confirms the benefit of MS$^+$.
When we have domain knowledge or experience from similar tasks that the variances tend to be small, we can tune the booster parameter to enjoy a better performance without worrying about suffering a linear regret when such information is not true.
In contrast, existing tuning techniques of UCB or TS~\cite{zhang16online} break their regret guarantees and risk suffering from a linear regret.

\begin{figure}
  \centering
  \begin{tabular}{cc}
    \includegraphics[width=.60\linewidth]{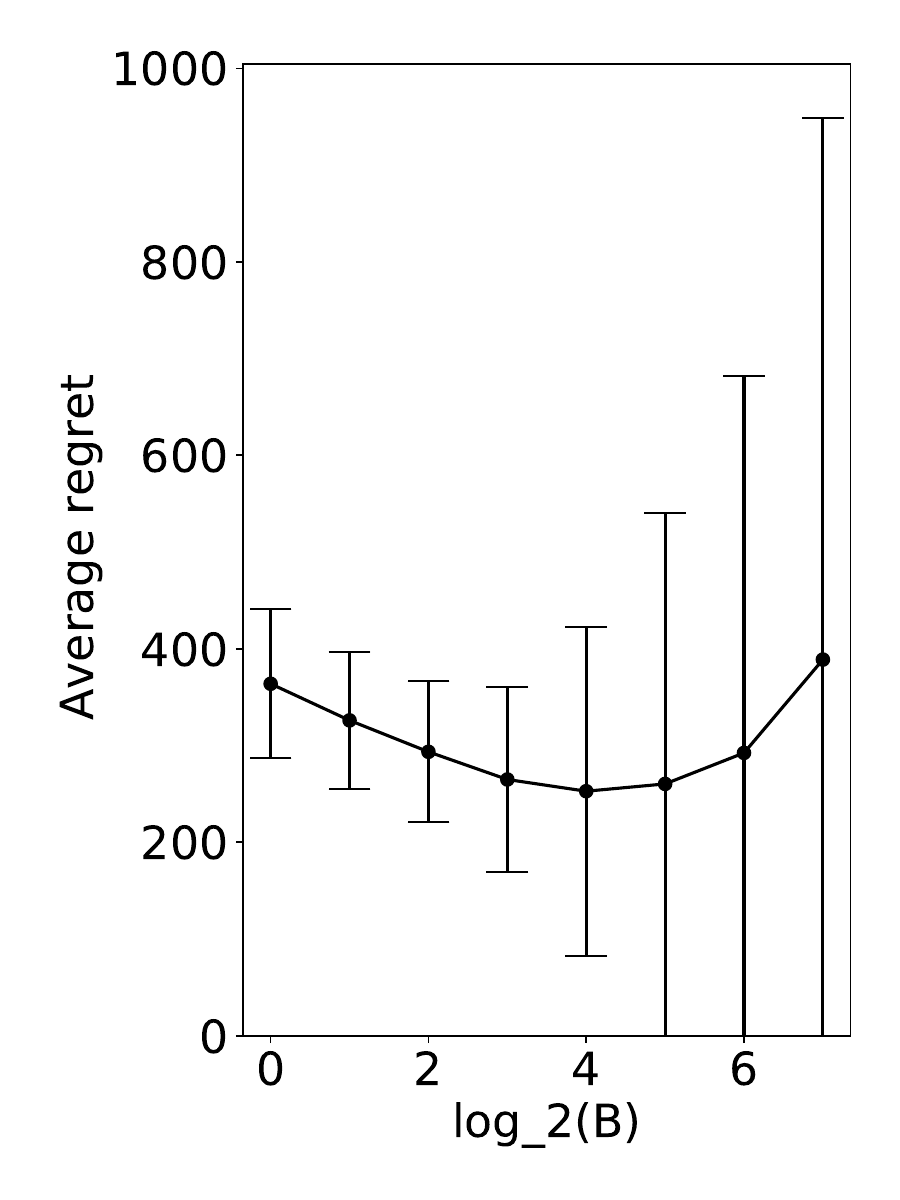}
  \end{tabular}
  \caption{The effect of tuning $B$ for MSP. The error bars are standard deviation over 200 trials.}
  \label{fig:expr-tuning}
\end{figure}

\begin{figure*}
  \centering
  \begin{tabular}{ccc}
    \includegraphics[width=.314\linewidth]{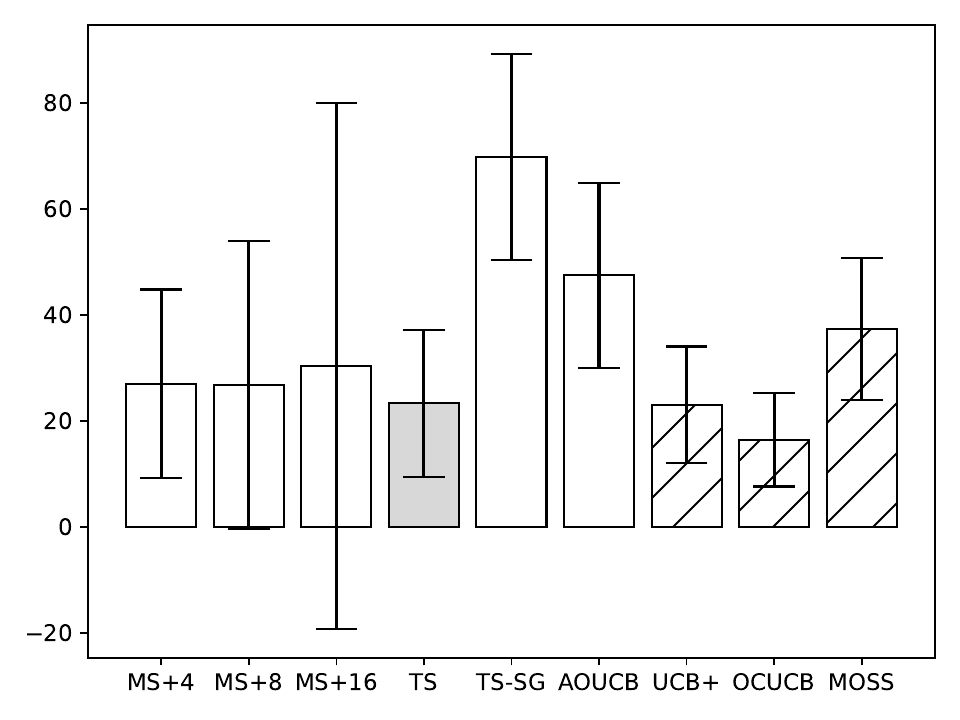} &
    \includegraphics[width=.314\linewidth]{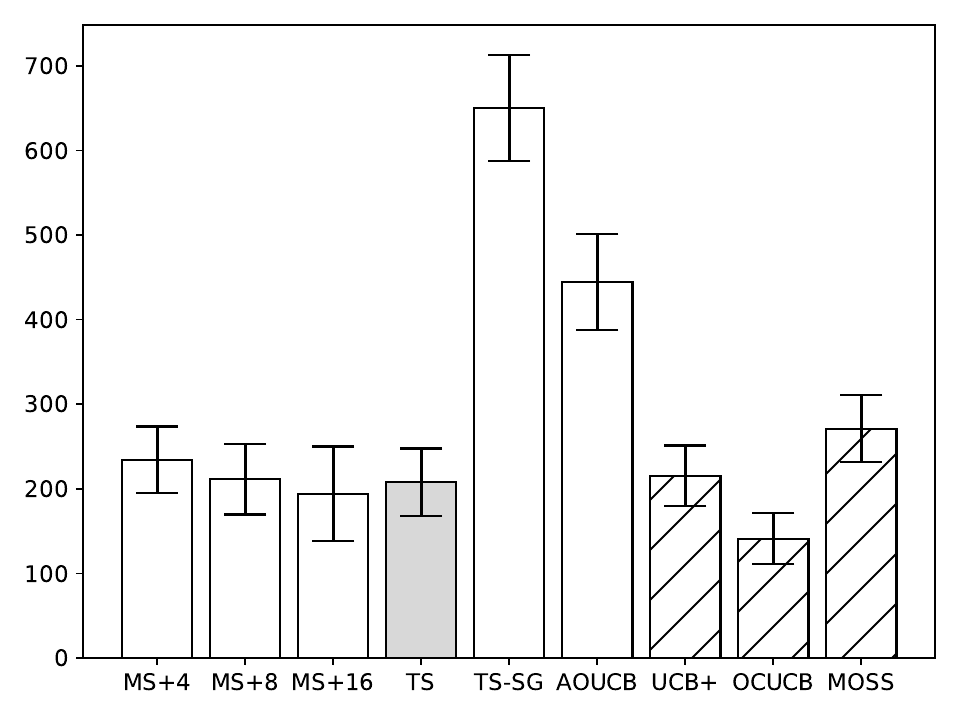} &
    \includegraphics[width=.314\linewidth]{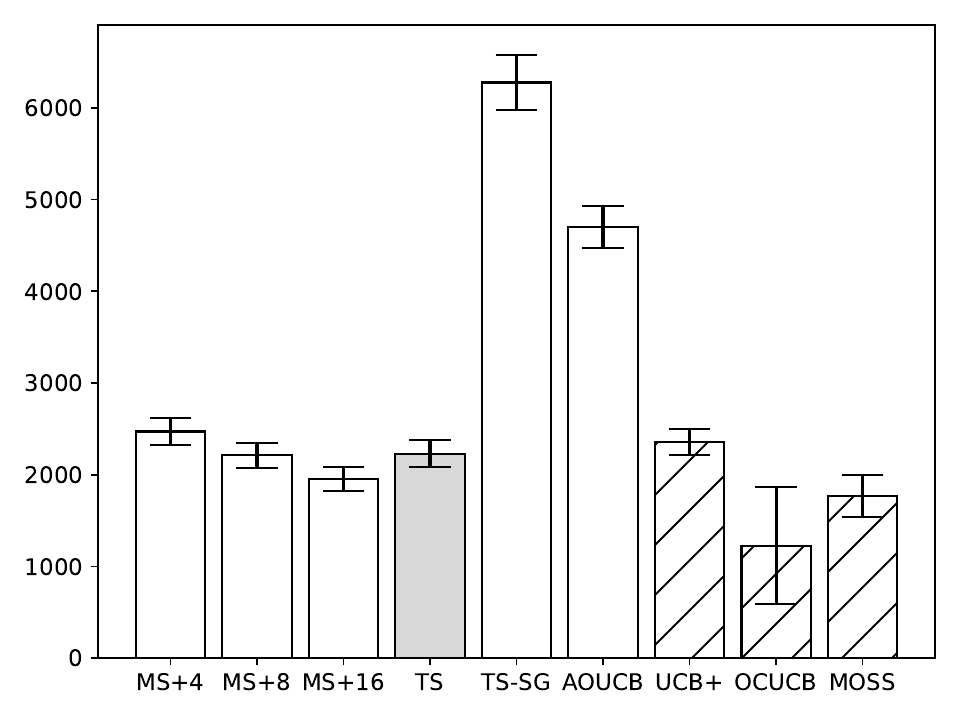} \\
    (a) Equal, $K=2$ & (b) Equal, $K=10$ & (c) Equal, $K=100$ \\
    \includegraphics[width=.314\linewidth]{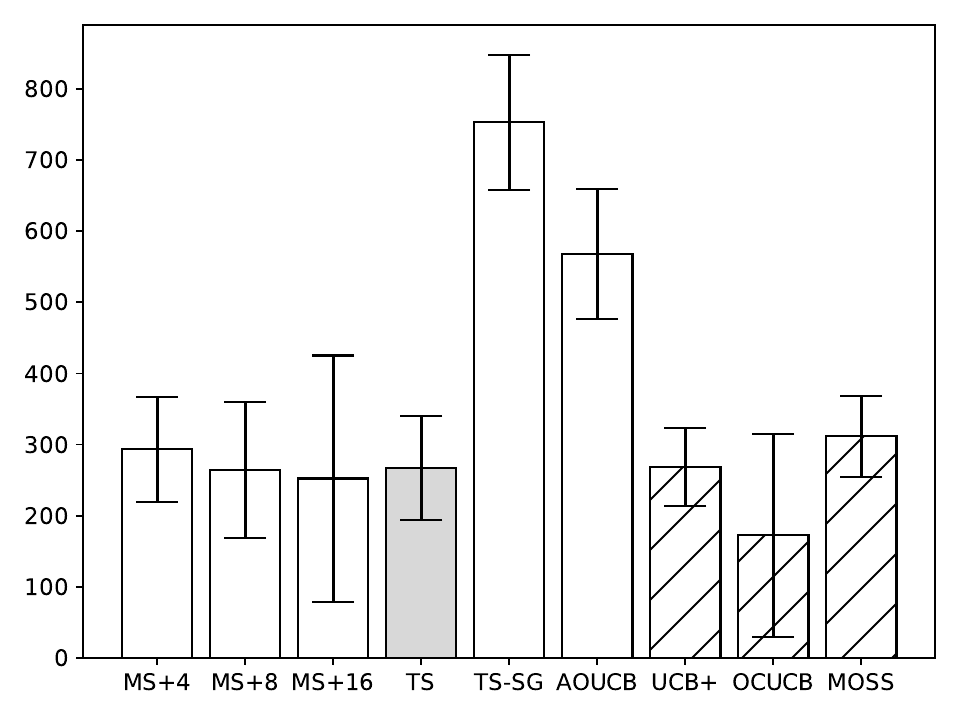}  &
    \includegraphics[width=.314\linewidth]{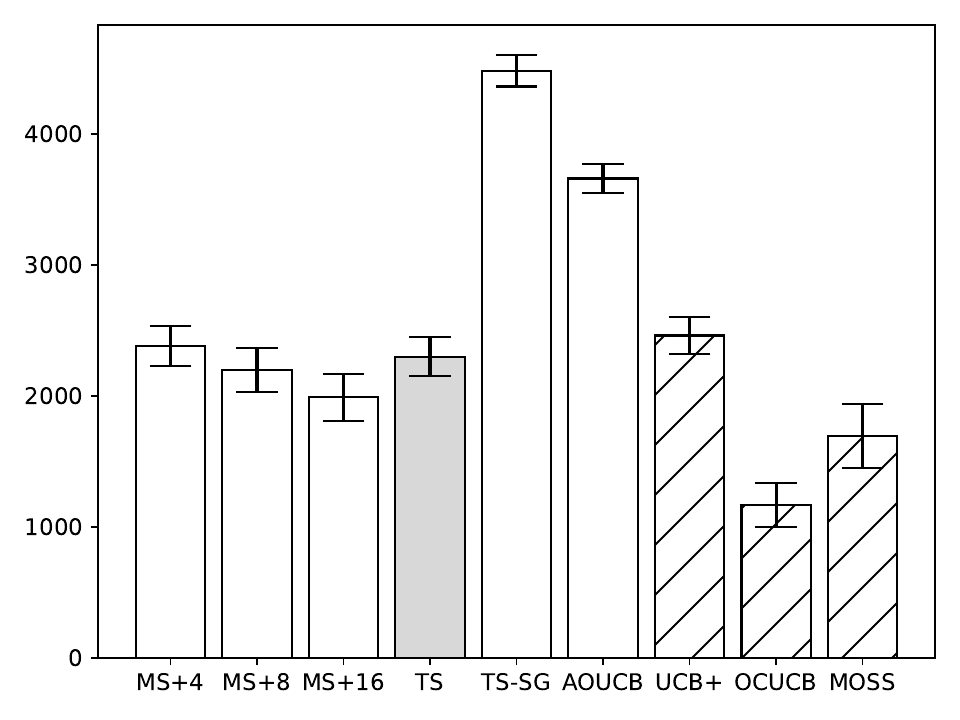} & \\
    (d) Linear, $K=10$ & (e) Linear $K=100$ \\
  \end{tabular}
  \caption{Experimental results for Gaussian environments from Table~\ref{tab:problems} with $T=20,000$ and $200$ trials. We show the average regrets as bars and the standard deviations with error bars. The shaded bar means that the method does not have a regret guarantee for sub-Gaussian rewards. The hatched bar means that the method is fixed-budget, which requires an extra information of $T$, making the comparison not exactly fair.}
  \label{fig:expr-g}
\end{figure*}
\begin{figure*}
  \centering
  \begin{tabular}{ccc}
    \includegraphics[width=.314\linewidth]{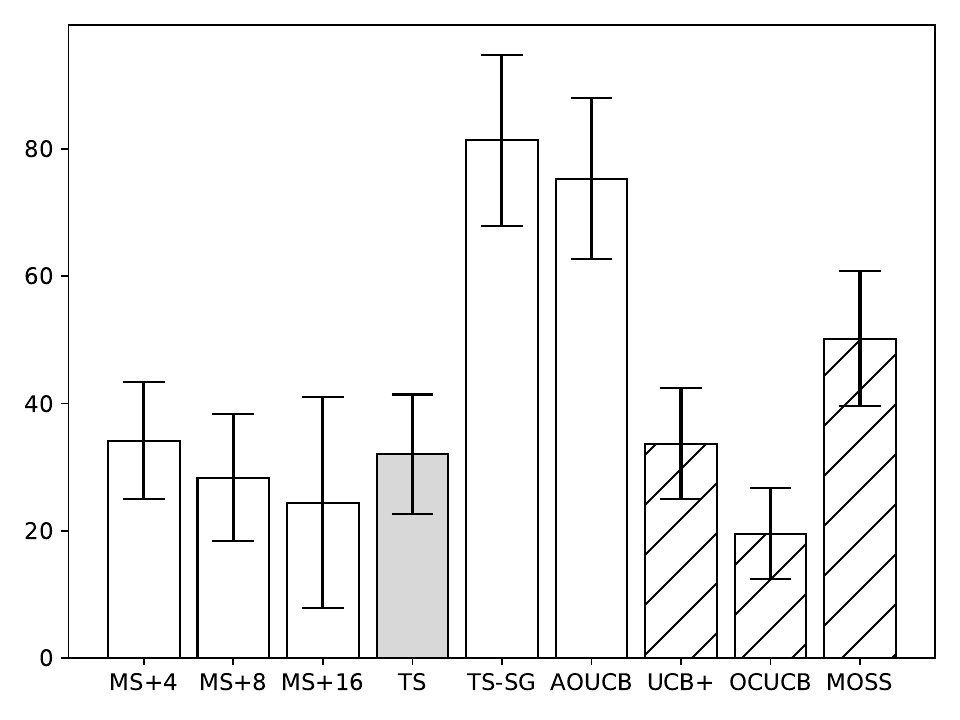} &
    \includegraphics[width=.314\linewidth]{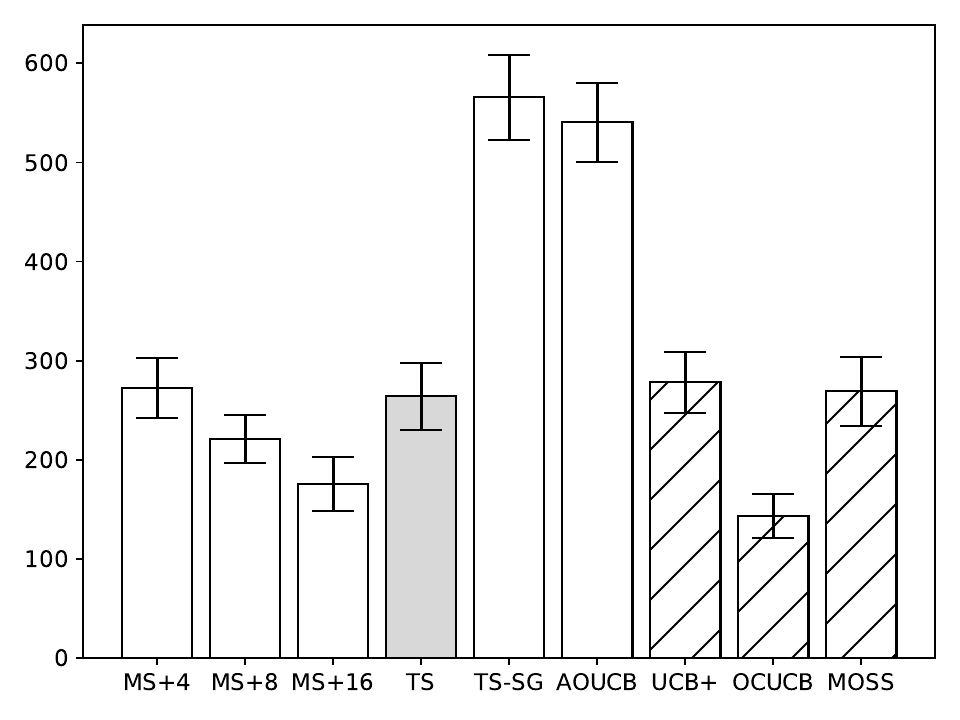} &
    \includegraphics[width=.314\linewidth]{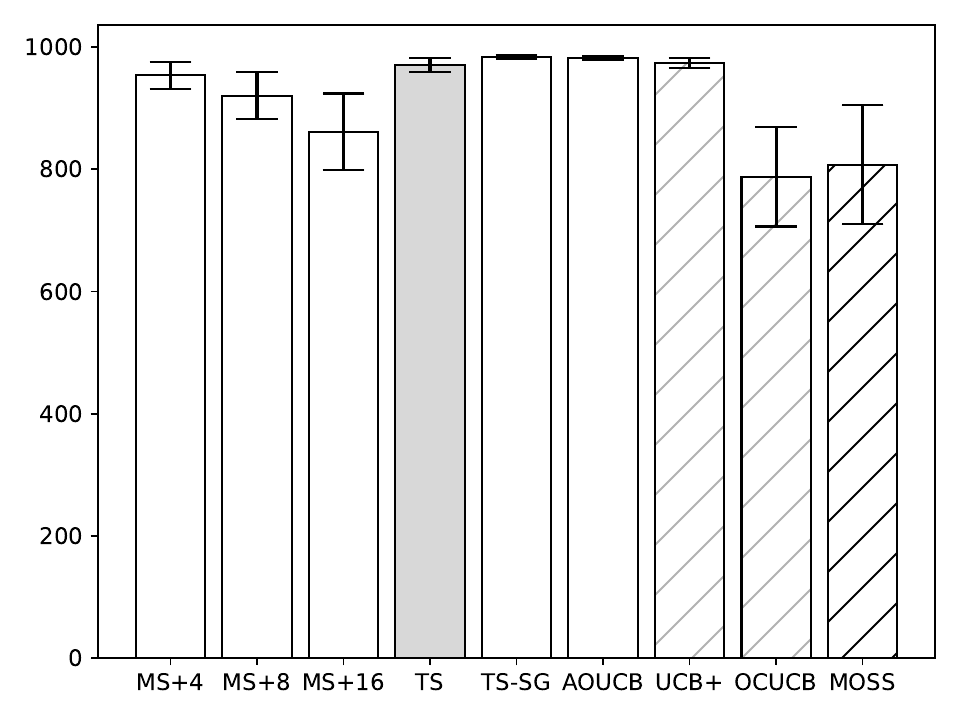} \\
    (a) Equal, $K=2$ & (b) Equal, $K=10$ & (c) Equal, $K=100$ \\
    \includegraphics[width=.314\linewidth]{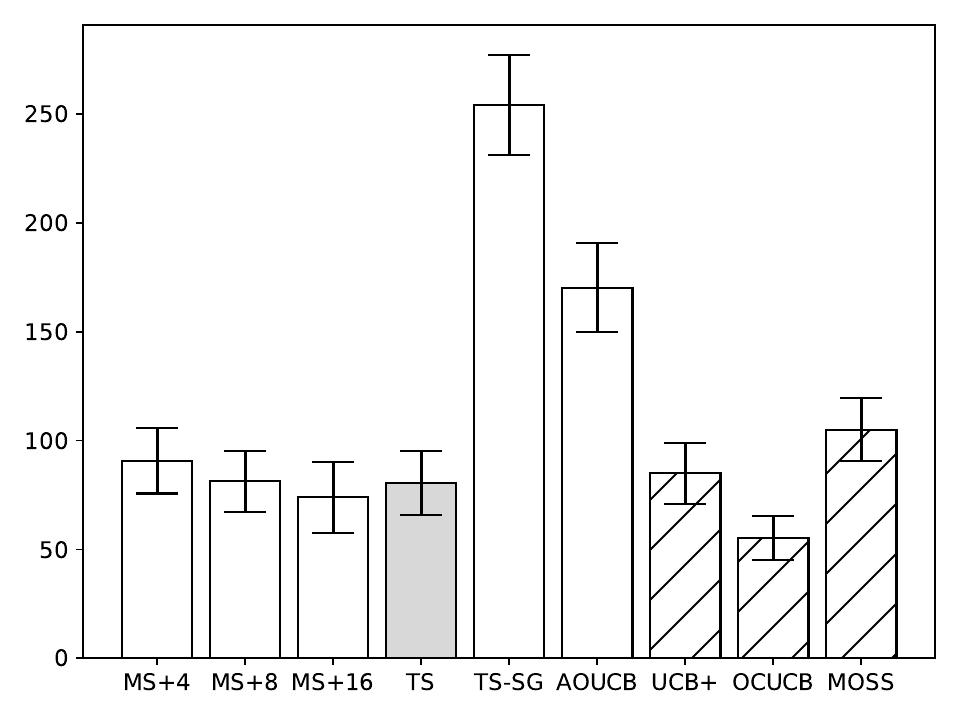}  &
    \includegraphics[width=.314\linewidth]{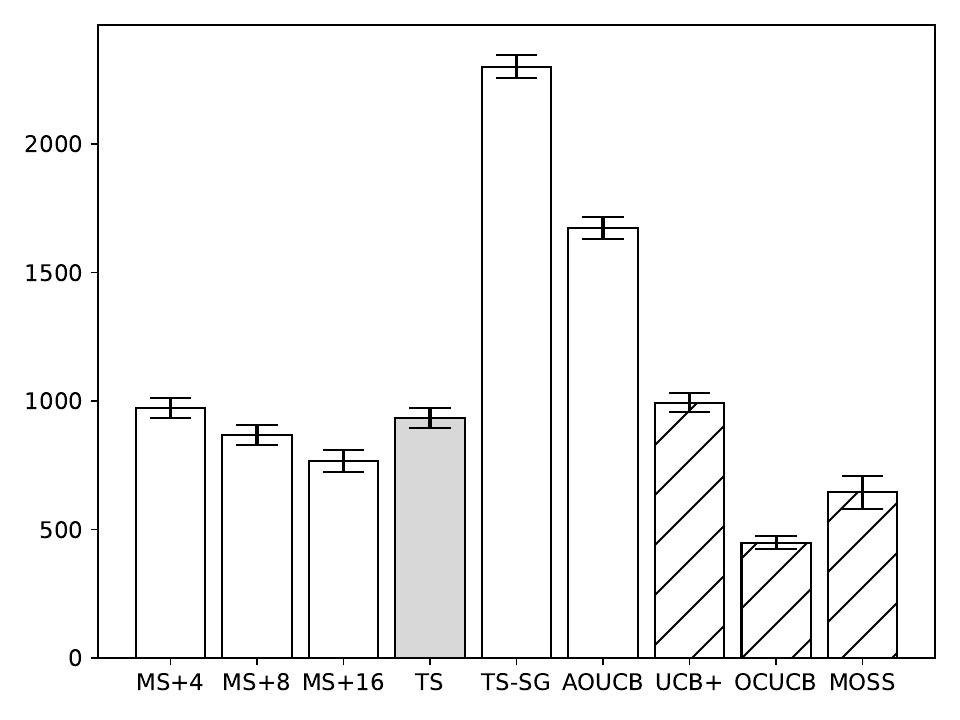} & \\
    (d) Linear, $K=10$ & (e) Linear $K=100$ \\
  \end{tabular}
  \caption{Experimental results for Bernoulli environments from Table~\ref{tab:problems}. }
  \label{fig:expr-b}
\end{figure*}

\section{CONCLUSION}
\label{sec:future}

We have revisited Maillard sampling, improved its regret bound, and proposed a variant that enjoys a better minimax regret bound.
As Boltzmann exploration is popular in reinforcement learning for its simplicity yet is provably suboptimal~\cite{cb17boltzmann},  we highly recommend that the practitioners try the following simplified version of MS as a replacement for Boltzmann exploration:
\begin{align}\label{eq:practical}
    p_{t,a} \propto B^{\one\{ a \in \arg\max_{b} \hat\mu_{t-1,b}\}} \exp(- \frac{\hat\Delta_{t-1,a}^2}{2\sigma^2} )
\end{align}
where both $B\ge 1$ is to be tuned and $\sigma^2$ can also be tuned when the sub-Gaussian parameter is unknown.
This version is simple enough to cause no friction for a quick implementation yet enjoys good guarantees in bandit problems.

Our work opens up numerous exciting research directions.
First, we can further investigate MS or even improve its mathematical properties.
For example, while \MSP allowed us to achieve the minimax ratio $\sqrt{\log(K)}$, it is not clear whether MS can achieve the same guarantee without further modifications.
The same issue happens with Thompson sampling -- the minimax ratio of $\sqrt{\log(K)}$ of Thompson sampling is achieved only after inflating the posterior probability~\cite[Theorem 1.3]{agrawal17nearoptimal}.
Also, we conjecture that MS is better than UCB under the delayed-reward setting since this is known to be true for randomized algorithms such as Thompson sampling~\cite{chapelle11anempirical}; see also \citet[Section 6]{dudik11efficient}.
It would be interesting to verify the conjecture theoretically and empirically.
Second, It would be interesting to further improve MS to match the regret guarantees of AdaUCB, which has the best-known regret bound in the literature.
Third, it would be interesting to generalize MS to the exponential family of distributions or to the linear bandit setting.
A natural candidate for the former is to take the MED algorithm~\cite{honda11asymptotically}, which is a generalization of MS as mentioned in related work, and perform finite time analyses.
Finally, while MS can be tuned to enjoy small expected regret, it tends to have a large variance.
It would be interesting to study the fundamental tradeoff between the expected pseudo regret and its variance.

\subsubsection*{Acknowledgements}
The authors would like to thank Tor Lattimore and Junya Honda for their valuable feedback and Nikos Karampatziakis for sharing Maillard sampling in social media.

\subsection*{References}
\bibliography{library-shared}

\clearpage
\appendix
\thispagestyle{empty}
\onecolumn \makesupplementtitle

%
%
\addcontentsline{toc}{section}{Appendix} 
\renewcommand \thepart{}
\renewcommand \partname{}
%
\part{} 

\vspace{-3em}
\parttoc 

\section{Related Algorithms}
\label{sec:relatedalg}

Beside Boltzmann exploration, there are a number of algorithms that shares a similar shape.
Among deterministic algorithms, IMED is the most similar one.
While IMED only has regret guarantees when rewards have (semi-)bounded support, IMED for the case of Gaussian that is reported in~\citet{lattimore18refining} has a striking similarity to MS:
\begin{align*}
  I_t = \arg \min_{i \in [K]} \frac{N_{t-1,i}}{2}\hDelta_{t-1,i}^2 + \log(N_{t-1,i})~.
\end{align*}
Except for the second term, the first term is exactly the negative of the exponent of MS.
In a similar vein, \citet{locatelli16anoptimal} propose an algorithm called APT for thresholding bandits.
While their problem is a pure exploration, their algorithm selects an arm by
\begin{align*}
  I_t = \arg \min_{i \in [K]} \sqrt{N_{t-1,i}} \hDelta'_{t-1,i}
\end{align*}
where $\hDelta'_{t-1,i} = |\hmu_{t-1,i} - \tau| + \eps$ where both $\tau$ and $\eps$ are given from the problem setup.
Except for the slight differences in the definition of $\hDelta'_{t-1,i}$, APT's index can be interpreted as the negative of MS's exponent.
Note, however, their estimated gap $\hDelta'_{t-1,i}$ is relative to a given problem parameter since the target reward level is given in thresholding bandit problems.
Thus, proof techniques from APT do not help analyzing MS that involves the gap based on the empirical best arm.
On the other hand, problem-dependent regret bounds necessarily include gaps $\Delta_i = \max_{j\in[K]} \mu_j - \mu_i$, which means that design of algorithms would be easier if we knew $\Delta_i$ since we can try to sample just enough number of times required by the lower bound.
Our results on MS reveals that one can directly use $\hDelta_i$ in the algorithm and perform a tight finite-time analysis of it.
We speculate that our regret analysis scheme can be useful for other similar bandit problems such as pure exploration.

Among many randomized stochastic bandit algorithms~\cite{kveton19garbage,kveton19perturbed,thompson33onthelikelihood,agrawal13further,kaufmann12thompson,korda13thompson}, Thompson sampling (TS) is the most popular one due to its simplicity and practical performance~\cite{chapelle11anempirical}.
While TS's sampling probabilities do not have a closed-form expression, the special case of $K=2$ arms reveals a connection between TS for Gaussian rewards and MS.
Suppose arm 1 has a higher mean reward than arm 2, and the reward noise is Gaussian with variance 1.
Let us use the Gaussian TS algorithm by~\citet{korda13thompson} whose posterior samples $X_1 \sim \cN(\hmu_1, 1/T_1)$ and $X_2 \sim \cN(\hmu_2, 1/T_2)$ where $T_1$ and $T_2$ are the pull counts of each arm.
Then, the probability of sampling arm 2 is $\PP(Z := X_2 - X_1 \ge 0)$.
One can show that $Z\sim \cN(\blue{-\hDelta} := \hmu_2 - \hmu_1, \fr{1}{T_1} + \fr{1}{T_2}  )$.
We can show that, with $\blue{\om^2} := \fr{1}{T_1} + \fr{1}{T_2} $,
\begin{align*}
  p_{t,2} = \PP(Z \ge 0) 
  &= \PP(\fr{Z-(-\hDelta)}{\om} \ge \fr{\hDelta}{\om}) 
  \\&=  1 - \Phi(\fr{\hDelta}{\om})
  \\&= \fr12 (1 - \text{erf}\del{\fr{\hDelta}{\sqrt{2}\om} } )
  \\&= \fr12 \text{erfc}\del{\fr{\hDelta}{\sqrt{2}\om} } 
  \\&\le \fr12 \fr{1}{\fr{\hDelta}{\om\sqrt{2}} \sqrt{\pi}} \exp\del{-\fr{\hDelta^2}{2\om^2} } 
\end{align*}
After a large enough iterations, note that $\om^2$ will be close to $1/T_2$ as most arm pulls will be on arm 1.
Then, except for the factor outside $\exp()$ above, the sampling probability almost coincides with MS.
Specifically, TS pulls arm 2 with probability $\frac{1}{2}\exp\del{-\fr{T_2\hDelta^2}{2} } $ roughly speaking and MS pulls arm 2 with probability like ${\exp({-\fr{T_2\hDelta^2}{2} })}/({1 + \exp({-\fr{T_2\hDelta^2}{2} })})$, which is asymptotically ${\exp\del{-\fr{T_2 \hDelta^2}{2} }}$.
Note that TS has $\text{erfc}(x)$ that decays slightly faster than $\exp(-x^2)$.
While this looks more efficient (i.e., less exploration), such a difference is rooted at the Gaussian reward assumption as opposed to the intrinsic benefit of the algorithmic framework. 
The MS algorithm presented in our paper is for sub-Gaussian rewards that include a much larger family of distributions, so MS must explore a bit more than Gaussian TS to maintain valid regret guarantees. 
Designing MS for Gaussian rewards is likely to result in a tighter exploration and possibly a contender to TS.

\section{Proofs for Maillard Sampling}
\label{Pfs_MS}

\subsection{Proof of Theorem~\ref{thm:ms}}
\label{sec:proof-ms}

We finish the proof by expanding on the term (F3).
Recall that (F3) is bounded by the following:
$$ \fr{4\sigma^2}{\eps^2} \sum_{q=1}^{\bar q - 1}  \frac{1}{q+\fr12} + \sum_{q=\bar q}^{\infty} \fr{1}{\exp( \fr{1}{4\sigma^2}(q+\fr12 )\eps^2) - 1}~.
$$
For the second term, note that, for $q \ge \bar q$, we have $\exp( \fr{1}{4\sigma^2}(q+\fr12 )\eps^2) \ge \exp(\fr{1}{4\sigma^2} (\bar q+\fr12 )\eps^2) =: \blue{\eta}$.
Using the fact that $X \ge \eta \implies \fr{1}{X-1} \le \fr{1}{1-\eta^{-1}} \fr1X$, 

\begin{align*}
  \sum_{q=\bar q }^{\infty} \fr{1}{\exp(\fr{1}{4\sigma^2}(q+\fr12)\eps^2) - 1}
  = \fr{1}{1-\eta^{-1}} \sum_{q=\bar q }^{\infty} \fr{1}{\exp(\fr{1}{4\sigma^2}(q+\fr12)\eps^2)}
    &= \fr{1}{1-\eta^{-1}}  \cd \fr{\exp(-\fr{1}{4\sigma^2}(\bar q + \fr12)\eps^2)}{1 - \exp(-\fr{1}{4\sigma^2}\eps^2)}
  \\&= \fr{1}{1-\eta^{-1}}  \cd \fr{\exp(-\fr{1}{4\sigma^2}(\bar q - \fr12)\eps^2)}{\exp(\fr{1}{4\sigma^2}\eps^2)-1}
  \\&\sr{(a)}{\le} 2 \cd \frac{\exp(-\fr{1}{8\sigma^2} \eps^2)}{\fr{1}{4\sigma^2}\eps^2}
  \le \frac{8\sigma^2}{\eps^2}
\end{align*}
where $(a)$ is by $e^z \ge z + 1$ and the choice of $\blue{\bar q} := \lt\lcl 1 \vee \del{\fr{4\sigma^2}{\eps^2} - \fr12}\rt\rcl$ that ensures $\eta \ge e$ and $\bar q \ge 1$.

For the first term, 
\begin{align*}
  \frac{4\sigma^2}{\eps^2}\sum_{q=1}^{\bar q - 1} \frac{1}{q+\fr12} 
  \le \frac{4\sigma^2}{\eps^2}\sum_{q=1}^{\bar q - 1} \frac{1}{q} 
  &\le \frac{4\sigma^2}{\eps^2}(1+\int_{q=1}^{\bar q - 1} \frac{1}{q} \diff q)
  \\&\le \frac{4\sigma^2}{\eps^2}\del{1 + \ln(\bar q - 1)} 
  \\&\le \frac{4\sigma^2}{\eps^2}\del{1 + \ln(\fr{4\sigma^2}{\eps^2} + \fr12)} ~.
\end{align*}
Hence
\begin{align*}
    \text{(F3)} &\le \fr{8\sigma^2}{\eps^2}+\frac{4\sigma^2}{\eps^2}\del{1 + \ln(\fr{4\sigma^2}{\eps^2} + \fr12)} 
    \\& =\frac{4\sigma^2}{\eps^2}\del{3 + \ln(\fr{4\sigma^2}{\eps^2} + \fr12)} 
    \\&=O\del{\fr{\sigma^2}{\eps^2}\del{1+\ln(1+\fr{\sigma^2}{\eps^2})}}~.
\end{align*}
Altogether with the choice of $u := \lt\lcl \fr{2\sigma^2(1+c)^2 \ln(\fr{T\Delta_a^2}{2\sigma^2})}{\Delta_a^2} \rt\rcl$, $\eps=\fr{c\Delta_a}{2(1+c)}$
\begin{align*}
    \EE [N_{T,a}]
    &\le u+ \text{(F1)}+\text{(F2)}+\text{(F3)}
    \\&= \lt\lcl \fr{2\sigma^2(1+c)^2 \ln(\fr{T\Delta_a^2}{2\sigma^2})}{\Delta_a^2} \rt\rcl+\fr{2\sigma^2}{\Delta_a^2} +\fr{2\sigma^2}{\del{\fr{c\Delta_a}{1+c} - \eps}^2}+O\del{\fr{\sigma^2}{\eps^2}\del{1+\ln(1+\fr{\sigma^2}{\eps^2})}}
    \\&=\lt\lcl \fr{2\sigma^2(1+c)^2 \ln(\fr{T\Delta_a^2}{2\sigma^2})}{\Delta_a^2} \rt\rcl+\fr{2\sigma^2}{\Delta_a^2}+\fr{8\sigma^2(1+c)^2}{c^2\Delta_a^2}+O\del{\fr{4(1+c)^2\sigma^2}{c^2\Delta_a^2}\del{1+\ln(1+\fr{4(1+c)^2\sigma^2}{c^2\Delta_a^2})}}
    \\&=\fr{2\sigma^2(1+c)^2 \ln(\fr{T\Delta_a^2}{\sigma^2} )}{\Delta_a^2} + O\del{1 \vee \del{\fr{(1+c)^2\sigma^2}{c^2\Delta_a^2}\ln\del{1+\fr{(1+c)^2\sigma^2}{c^2\Delta_a^2}}}}~.
\end{align*}
Hence
\begin{align*}
  \EE\Reg_T
  &=\sum_{a\in[K]: \Delta_a > 0}^{} \Delta_a \cd \EE [N_{T,a}]
\\&\le \sum_{a\in[K]: \Delta_a > 0}^{}\Bigg(\fr{2\sigma^2(1+c)^2 \ln(\fr{T\Delta_a^2}{\sigma^2} )}{\Delta_a}+ O\del{\Delta_a \vee \del{\fr{(1+c)^2\sigma^2}{c^2\Delta_a}\ln\del{1+\fr{(1+c)^2\sigma^2}{c^2\Delta_a^2}}}}\Bigg)~.
\end{align*}

\subsection{Proof of Corollary \ref{thm1_cor}}

\begin{proof}
From Section~\ref{sec:proof-ms} we have: MS satisfies that, \kjold{actually we should have ``$\vee 1$'' in the first term below} $\forall T\ge1$ and $c>0$, 
\begin{align*}
 \EE[N_{T,a}]\le \fr{2\sigma^2(1+c)^2 \ln(\fr{T\Delta_a^2}{\sigma^2} )}{\Delta_a^2} + O\del{1 \vee \del{\fr{(1+c)^2\sigma^2}{c^2\Delta_a^2}\ln\del{1+\fr{(1+c)^2\sigma^2}{c^2\Delta_a^2}}}}~.
\end{align*}
Therefore,
\begin{align*}
    \EE \Reg_T&=\sum_{a\in[K]:\Delta_a >0} \Delta_a \EE[N_{T,a}]
    \\&=\sum_{a\in[K]:\Delta_a<\Delta}\Delta_a \EE[N_{T,a}]+
        \sum_{a\in[K]:\Delta_a\ge \Delta}\Delta_a \EE[N_{T,a}] 
    \\&<T\Delta + K\del{\fr{2\sigma^2(1+c)^2 \ln(\fr{T\Delta^2}{\sigma^2})}{\Delta} + O\del{\fr{(1+c)^2\sigma^2}{c^2\Delta}\ln\del{1+\fr{(1+c)^2\sigma^2}{c^2\Delta^2}}}}+\sum_{a\in[K]:\Delta_a >0} O(\Delta_a)~.
\end{align*}
\kjold{actually, a case-by-case study needed. $\Delta^{-1}\ln(T\Delta^2)$ is decreasing only when $T\Delta^2 \le e^2$. So, for the regime that it is not decreasing, we should have $\Delta \del{1 \vee \del{\fr{(1+c)^2\sigma^2}{c^2\Delta_a^2}\ln\del{1+\fr{(1+c)^2\sigma^2}{c^2\Delta_a^2}}}}$ }
Then by choosing $\Delta=\Theta(\sigma\sqrt{\fr{K\ln(T)}{T}})$, the upper bound will become
\begin{align*}
    \EE \Reg_T&<O(\sigma\sqrt{KT\ln(T)})+O(\sigma\sqrt{KT\ln(T)})+\sum_{a\in[K]:\Delta_a >0} O(\Delta_a)
    \\&=O(\sigma\sqrt{KT\ln(T)})~.
\end{align*}
\end{proof}

\section{Proofs for Maillard Sampling$^+$}
\label{Pfs_MSP}

\subsection{Proof of Theorem~\ref{thm_IMS}}
\begin{proof}
Just like the proof of MS, it suffices to upper bound $\EE[N_a]$ for arm $a$ such that $\Delta_a >0$.
Let $T \ge K$.
First, using the definition of the initialization step, for all $u > 0$ we have
\begin{align*}
  \EE[N_{T,a}]&=
  \EE[N_{K,a}]+\EE[N_{T,a}-N_{K,a}]
\\&\le 1+v + \EE\sbr{\sum_{t=K+v+1}^T \one\cbr{ I_t = a \text{ and } N_{t-1,a} > v}}
  \\&=   1+v + \EE\sbr{\sum_{t=K+v}^{T-1} \one\cbr{ I_{t+1} = a \text{ and } N_{t,a} > v}}
  \\&=u+ \EE\sbr{\sum_{t=K-1+u}^{T-1} \one\cbr{ I_{t+1} = a \text{ and } N_{t,a} \ge u}}~.
\end{align*}
Let $\blue{\hmu_{t,\max}} = \max_{a} \hmu_{t,a}$.
Let $c > 0$.
Then, we consider the following branching conditions under $\cbr{I_{t+1} = a, N_t(a) \ge u}$: 
\begin{itemize}
    \item (F1) $\hmu_{t,\max} \ge \mu_1 - \eps, \hDelta_{t,a} > \fr{\Delta_a}{1+c}$ 
    \item (F2) $\hmu_{t,\max} \ge \mu_1 - \eps, \hDelta_{t,a} \le \fr{\Delta_a}{1+c}$
    \item (F3) $\hmu_{t,\max} < \mu_1 - \eps~.$ 
\end{itemize}

For case (F1), by definition of $p_{t+1}$ and $\blue{u} := \lt\lcl \fr{2\sigma^2(1+c)^2 \ln(\fr{T\Delta_a^2}{2\sigma^2}(1+\ln(1+\fr{T\Delta_a^2}{2\sigma^2})))}{\Delta_a^2} \rt\rcl$, we have 
\begin{align*}
    &\EE \sum_{t=K-1+u}^{T-1} \one\cbr{I_{t+1} = a,~ N_{t,a} \ge u,~ \hmu_{t,\max} \ge \mu_1 - \eps,~ \hDelta_{t,a} > \fr{\Delta_a}{1+c}}
  \\&\le  \sum_{t=K-1+u}^{T-1} \PP\del{  I_{t+1} = a \mid N_{t,a} \ge u,~ \hDelta_{t,a} > \fr{\Delta_a}{1+c}} \PP\del{ N_{t,a} \ge u,~ \hDelta_{t,a} > \fr{\Delta_a}{1+c}} 
  \\&\stackrel{(a)}{\le}  \sum_{t=K-1+u}^{T-1}   \fr 1 B \exp\del{-\fr{1}{2\sigma^2}u\del{\fr{\Delta_a}{1+c}}^2}\cd (1+\fr{D}{2\sigma^2}u\del{\fr{\Delta_a}{1+c}}^2) \cd 1
  \end{align*}
where $(a)$ is by the fact that for $D\in\lparen0,1\rbrack$, $(1+Dx)\cd \exp(-x)$ is decreasing in $x>0$.
Let $\blue{A} = \fr{2\sig^2 }{\Delta_a^2}$. 
Then, $u \ge (1+c)^2 A \ln(\fr T A ( 1 + \ln( 1 + \fr T A)))$.
The last line above is then bounded by 
  \begin{align*}
  &\fr T B \exp\del{-\fr{1}{A(1+c)^2}u}\cd (1+\fr{D}{A(1+c)^2}u) \cd 1
  \\&\le \fr{1}{B} \frac{A}{1 + \ln(1 + \fr T A)} \cd \del{1 + D\ln(\fr T A(1 + \ln(1 + \fr T A))  }
  \\&\le \fr{1}{B}\cd A \cd 2D
  \\&=   \fr{2D}{B} \cd \fr{2\sigma^2}{\Delta_a^2}~.
\end{align*}

For case (F2), note that 
\begin{align*}
  \cbr{\hmu_{t,\max} \ge \mu_1 - \eps} \cap \cbr{ \hDelta_{t,a} \le \fr{1}{1+c} \Delta_a} \subseteq \cbr{\hmu_{t,a} - \mu_a \ge \fr{c}{1+c}\Delta_a - \eps}~.
\end{align*}
Let  $\blue{\tau_{a}(k) }$ be the time step $t$ such that the arm a was pulled at $t$ and the number of arm pulls of arm a becomes $k$ at the end of the time step $t$; i.e., $\tau_a(k) = \min\cbr{t\ge1: N_{t,a} = k}$. 
We use the shortcut $\blue{\hmu_{a,(k)}} := \hmu_{\tau_a(k),a}$ to denote the sample mean of arm $a$ after pulling it for the $k$-th time, . 
Thus,
\begin{align*}
    &\EE \sum_{t=K-1+u}^{T-1} \one\cbr{I_{t+1} = a,~ N_{t,a} \ge u,~ \hmu_{t,\max} \ge \mu_1 - \eps,~ \hDelta_{t,a} \le \fr{\Delta_a}{1+c}}
  \\&\le \sum_{t=K-1+u}^{T-1} \PP\del{I_{t+1} = a,~ \hmu_{t,a} - \mu_a \ge \fr{c\Delta_a}{1+c} - \eps}
  \\&\le \EE \sbr{\sum_{k=1}^{\infty} \sum_{t=\tau_1(k)}^{\tau_1(k+1)-1} \one \cbr{I_{t+1} = a} \cd \one \cbr{ \hmu_{t,a} - \mu_a \ge \fr{c\Delta_a}{1+c} - \eps}}
  \\&= \EE \sbr{\sum_{k=1}^{\infty}  \one \cbr{ \hmu_{a,(k)} - \mu_a \ge \fr{c\Delta_a}{1+c} - \eps} \sum_{t=\tau_1(k)}^{\tau_1(k+1)-1}\one \cbr{I_{t+1} = a} }
  \\&= \EE \sbr{\sum_{k=1}^{\infty}  \one \cbr{ \hmu_{a,(k)} - \mu_a \ge \fr{c\Delta_a}{1+c} - \eps}  }
  \\&= \sum_{k\ge1} \PP\del{ \hmu_{a,(k)} - \mu_a \ge \fr{c\Delta_a}{1+c} - \eps} \tag{$ N_{t,a}\ge 1, \forall t\ge K$, since $u>1$}
  \\&\le \sum_{k\ge1} \exp\del{-\fr{1}{2\sigma^2}k\del{\fr{c\Delta_a}{1+c} - \eps}^2}
   ~= \fr{1}{\exp\del{\fr{1}{2\sigma^2}\del{\fr{c\Delta_a}{1+c} - \eps}^2} - 1}
   ~\le \fr{2\sigma^2}{\del{\fr{c\Delta_a}{1+c} - \eps}^2}~.
\end{align*}

We now consider case (F3).
This is the case where the best arm behaves badly.
Let $\blue{\cF_t}$ be the $\sigma$-algebra generated by $(A_1,R_1,\ldots,A_{t},R_{t})$.
We define the shortcut $\blue{\PP_t(\cE)} := \PP(\cE \mid \cF_t)$ and similarly $\blue{\EE_t[X]} = \EE[X\mid \cF_t]$.

Recall $B\ge1$ and $C,D\in(0,1]$.
Let $Z = \sum_a \hp_{t+1,a}$.
Note that, for every $a\in[K]$, the following is true regardless of whether $a$ is the empirical best arm or not:
\begin{align*}
   \fr1Z \exp\del{-\fr{1}{2\sigma^2}N_{t,a}\hDelta^2_{t,a}+\ln(1+\fr{D}{2\sigma^2} N_{t,a}\hDelta^2_{t,a})} 
   \le \PP_t(I_{t+1}=a) 
   \le \frac{B(1 + C\ln(1 + \ln(\frac{t}{N_{t,a}})))}{Z}~.
\end{align*}
This helps us relate the probability of pulling arm $a$ with arm 1.
Thus, under the event $\hmu_{t,\max} < \mu_1 - \eps$, we have
\begin{align}
  \begin{aligned}
      \PP_t(I_{t+1}=a) 
   \le \fr{B \cd (1 + C\ln(1 + \ln(\frac{t}{N_{t,a}})))}{\exp(-\fr{1}{2\sigma^2}N_{t,1} \cd (\mu_1-\eps-\hmu_{t,1})_+^2)
    \cd (1+\fr{D}{2\sigma^2}N_{t,1}\cd (\mu_1-\eps-\hmu_{t,1})_+^2)
    } \PP_t(I_{t+1}=1)~.
  \end{aligned}
\end{align}

Let $\blue{\tau_{1}(k) }$ be the time step $t$ such that the arm 1 was pulled at $t$ and the number of arm pulls of arm 1 becomes $k$ at the end of the time step $t$; i.e., $\tau_1(k) = \min\cbr{t\ge1: N_{t,1} = k}$. 
We use the shortcut $\blue{\hmu_{1,(k)}} := \hmu_{\tau_1(k),1}$. 
With this, we can bound (A) as follows:
\begin{align*}
  &\EE \sum_{t=K-1+u}^{T-1} \one\cbr{I_{t+1} = a,~ N_{t,a} \ge u,~ \hmu_{t,\max} < \mu_1 - \eps}
  \\&\le \EE\sbr{ \sum_{t=K-1+u}^{T-1} \EE_t\sbr{\one\cbr{I_{t+1} = a,~ \hmu_{t,\max} < \mu_1 - \eps} }  }
  \\&= \EE\sbr{ \sum_{t=K-1+u}^{T-1} \one\cbr{ \hmu_{t,\max} < \mu_1 - \eps} \EE_t\sbr{\one\cbr{I_{t+1} = a} }  } 
  \\&\le \EE\sbr{ \sum_{t=K-1+u}^{T-1} \one\cbr{ \hmu_{t,\max} < \mu_1 - \eps}\cd \fr{\exp\del{\fr{1}{2\sigma^2} N_{t,1} \cd (\mu_1-\eps-\hmu_{t,1})^2}}{1+\fr{D}{2\sigma^2}\cd N_{t,1}\cd (\mu_1-\eps-\hmu_{t,1})^2}\cd \del{B\cd(1+C\ln(1+\ln(\fr{t}{N_{t,a}})))} \EE_t\sbr{ \one\cbr{I_{t+1} = 1} } } 
  \\&= \EE\sbr{ \sum_{t=K-1+u}^{T-1} \EE_t \sbr{\one\cbr{ \hmu_{t,\max} < \mu_1 - \eps}\cd \fr{\exp\del{\fr{1}{2\sigma^2} N_{t,1} \cd (\mu_1-\eps-\hmu_{t,1})^2}}{1+\fr{D}{2\sigma^2}\cd N_{t,1}\cd (\mu_1-\eps-\hmu_{t,1})^2} \cd \del{B\cd(1+C\ln(1+\ln(\fr{t}{N_{t,a}})))} \cd \one\cbr{I_{t+1} = 1}  }}
  \\&= \EE\sbr{ \sum_{t=K-1+u}^{T-1} \one\cbr{ \hmu_{t,\max} < \mu_1 - \eps}\cd \fr{\exp\del{\fr{1}{2\sigma^2}N_{t,1} \cd (\mu_1-\eps-\hmu_{t,1})^2}}{1+\fr{D}{2\sigma^2}\cd N_{t,1}\cd (\mu_1-\eps-\hmu_{t,1})^2}
  \cd \del{B\cd(1+C\ln(1+\ln(\fr{t}{N_{t,a}})))}\cd \one\cbr{I_{t+1} = 1}  }
  \\&\le \EE\sbr{ \sum_{t=K-1+u}^{T-1} \one\cbr{ \hmu_{t,1} < \mu_1 - \eps}\cd \fr{\exp\del{\fr{1}{2\sigma^2}N_{t,1} \cd (\mu_1-\eps-\hmu_{t,1})^2}}{1+\fr{D}{2\sigma^2}\cd N_{t,1}\cd (\mu_1-\eps-\hmu_{t,1})^2} \cd \del{B\cd(1+C\ln(1+\ln(\fr{t}{N_{t,a}})))}\cd
  \one\cbr{I_{t+1} = 1}  }
  \\&= \EE\sbr{ \sum_{k=1}^{\infty} \sum_{t=\tau_1(k)}^{\tau_1(k+1)-1} \one\cbr{ \hmu_{1,(k)} < \mu_1 - \eps}\cd \fr{\exp\del{ \fr{1}{2\sigma^2} k \cd (\mu_1-\eps-\hmu_{1,(k)})^2}}{1+\fr{D}{2\sigma^2}\cd k\cd (\mu^*-\eps-\hmu_{1,(k)})^2}
  \cd \del{B\cd(1+C\ln(1+\ln(\fr{T}{u})))}\cd \one\cbr{I_{t+1} = 1} }
  \\&\le \EE\sbr{ \sum_{k=1}^{\infty}  \one\cbr{ \hmu_{1,(k)} < \mu_1 - \eps}\cd \fr{\exp\del{\fr{1}{2\sigma^2} k \cd (\mu_1-\eps-\hmu_{1,(k)})^2}}{1+\fr{D}{2\sigma^2}\cd k\cd (\mu_1-\eps-\hmu_{1,(k)})^2}
  \cd \del{B\cd(1+C\ln(1+\ln(\fr{T}{u})))} \sum_{t=\tau_1(k)}^{\tau_1(k+1)-1} \one\cbr{I_{t+1} = 1}  }
  \\&= \EE\sbr{ \sum_{k=1}^{\infty}  \one\cbr{ \hmu_{1,(k)} < \mu_1 - \eps}\cd \fr{\exp\del{\fr{1}{2\sigma^2} k \cd (\mu_1-\eps-\hmu_{1,(k)})^2}}{1+\fr{D}{2\sigma^2}\cd k\cd (\mu_1-\eps-\hmu_{1,(k)})^2}\cd \del{B\cd(1+C\ln(1+\ln(\fr{T}{u})))} }
  \\&= \EE\sbr{ \sum_{k=1}^{\infty}  \one\cbr{ \blue{\bar X_k} > \eps}\cd \fr{\exp\del{\fr{1}{2\sigma^2} k \cd (\blue{\bar X_k} - \eps)^2}}{1+\fr{D}{2\sigma^2}k\cd (\blue{\bar X_k} - \eps)^2} \cd B \cd (1 + C\ln(1 + \ln(\frac{T}{u}))) } 
\end{align*}
where the last line is by defining $\blue{\barX_k} := \mu_1 - \hmu_{1,(k)}$.

We bound the following term:
\begin{align*}
  &  \EE\sbr{ \sum_{k=1}^{\infty}  \one\cbr{ {\bar X_k} > \eps}\cd \fr{\exp\del{\fr{1}{2\sigma^2} k \cd ({\bar X_k} - \eps)^2}}{1+\fr{D}{2\sigma^2}k\cd ({\bar X_k} - \eps)^2} } 
  \\&\le \EE\sbr{ \sum_{k=1}^{\infty}  \sum_{q=1}^\infty \one\cbr{ \mu_1 - \hmu_{1,(k)} \in \lt\lparen \eps+\fr{q-1}{2}\eps,~ \eps + \fr{q}{2} \eps \rt\rbrack  } \cd \fr{\exp\del{\fr{1}{2\sigma^2} k \cd \del{\fr{q\eps}{2}}^2}}{1+\fr{D}{2\sigma^2}k(\fr{q\eps}{2})^2} }
  \\&=  \sum_{k=1}^{\infty}  \sum_{q=1}^\infty \PP\del{ \mu_1 - \hmu_{1,(k)} \in \lt\lparen \eps+\fr{q-1}{2}\eps,~ \eps + \fr{q}{2} \eps \rt\rbrack  } \cd \fr{\exp\del{\fr{1}{2\sigma^2} k \cd \del{\fr{q\eps}{2}}^2}}{1+\fr{D}{2\sigma^2}k(\fr{q\eps}{2})^2} 
  \\&\le  \sum_{k=1}^{\infty}  \sum_{q=1}^\infty \exp\del{-\fr{1}{2\sigma^2}k \del{\fr{q+1}{2} \eps}^2} \cd \fr{\exp\del{\fr{1}{2\sigma^2}k \cd \del{\fr{q\eps}{2}}^2}}{1+\fr{D}{2\sigma^2}k(\fr{q\eps}{2})^2}
  \\&=    \sum_{k=1}^{\infty}  \sum_{q=1}^\infty \fr{\exp \del{-\fr{1}{4\sigma^2}(q+\fr{1}{2})k\eps^2}}{1+\fr{D k q^2 \eps^2}{8\sigma^2}}
   \\&<\sum_{k=1}^{\infty}  \sum_{q=1}^\infty \fr{\exp \del{-\fr{1}{4\sigma^2}(q+\fr{1}{2})k\eps^2}}{1\vee\fr{D k q^2 \eps^2}{8\sigma^2}}
   \\&=\sum_{q=1}^{\infty} \del{ \sum_{k=1}^{k<\fr{8\sigma^2}{D q^2 \eps^2}} \exp \del{-\fr{1}{4\sigma^2}(q+\fr{1}{2})k\eps^2}+\sum_{k\ge\fr{8\sigma^2}{D q^2 \eps^2}}^{\infty} \fr{8\sigma^2\exp \del{-\fr{1}{4\sigma^2}(q+\fr{1}{2})k\eps^2}}{D k q^2 \eps^2}} ~.
\end{align*}

For the first term,
let us define $\blue Q = \fr{8\sig^2}{D\eps^2}$ and $\blue R = \fr{\eps^2}{4\sig^2}$.
Then,
\begin{align*}
    \sum_{q=1}^{\infty}  \sum_{k=1}^{k<\fr{Q}{q^2}} \exp \del{-R(q+\fr{1}{2})k}
    &\le \sum_q \frac{1-\exp(-R(q+ \fr12)\fr{Q}{q^2})}{\exp(R(q+\fr12)) - 1}
    \\&\sr{(a)}{\le} Q\sum_q \fr{1}{q^2}
    \\&= 2Q 
\end{align*}
where $(a)$ is by $e^z \ge z + 1$.

For the second term,
\begin{align*}
    \sum_{q=1}^{\infty} \sum_{k\ge Q/q^2} \fr{Q}{k q^2} \exp \del{-R(q+\fr{1}{2})k}
    \le  \underbrace{\sum_{q=1}^{\infty}\sum_{k\in \lbrack Q/q^2, Q/q\rbrack} \fr{Q}{k q^2} \exp \del{-R(q+\fr{1}{2})k} }_{\textstyle =: G_1} + \underbrace{\sum_{q=1}^{\infty} \sum_{k> Q/q} \fr{Q}{k q^2} \exp \del{-R(q+\fr{1}{2})k} }_{\textstyle =: G_2}~.
\end{align*}
We bound $G_1$ as follows:
\begin{align*}
    G_1
    &\le \sum_{q=1}^{\infty}\sum_{k\in \lbrack Q/q^2, Q/q\rbrack} \fr{Q}{k q^2}
    \\&\stackrel{(a)}{\le} \sum_{q=1}^{\infty} \fr{Q}{q^2} \del{ 1 +  \ln(\frac{Q/q}{Q/q^2}) }
    \\&= \sum_{q=1}^{\infty} \fr{Q}{q^2} \del{ 1 +  \ln(q)}
    \\&= Q(\fr{\pi^2}{6} + \sum_{q=1}^\infty \frac{\ln(q)}{q^2})    
    \\&\le 3 Q
\end{align*}
where $(a)$ is by the fact that 
\begin{align}\label{eq:sum-by-integral}
  \sum_{i=a}^b f(i) \le \max_{x\in[a,b]} f(x) + \int_{a}^b f(x) \diff x  ~.
\end{align}
We bound $G_2$ as follows:
\begin{align*}
    G_2
    &\le \sum_{q=1}^{\infty} \sum_{k> Q/q} \fr{Q}{k q^2} \exp \del{-R(q+\fr{1}{2})k} 
    \\&\le \sum_{q=1}^{\infty} \fr{1}{q} \sum_{k> Q/q}  \exp \del{-R(q+\fr{1}{2})k}
    \\&\le \sum_{q=1}^{\infty} \fr{1}{q} \frac{\exp(-R(q+\fr12)((1 \vee \fr Q q)-1) )}{\exp(R(q+ \fr 1 2)) - 1} \tag{note $k\ge1$}
    \\&\le \sum_{q=1}^{\infty} \fr{1}{q} \frac{1}{R(q+\fr 1 2)} \tag{by $e^z \ge z-1$}
    \\&\le \fr 2 R~.
\end{align*}
Thus,
\begin{align*}
    \EE\sbr{ \sum_{k=1}^{\infty}  \one\cbr{ {\bar X_k} > \eps}\cd \fr{\exp\del{\fr{1}{2\sigma^2} k \cd ({\bar X_k} - \eps)^2}}{1+\fr{D}{2\sigma^2}k\cd ({\bar X_k} - \eps)^2}  }
    \le \del{\frac{40}{D} + 8} \fr{\sig^2}{\eps^2}~.
\end{align*}

Hence \begin{align*}
   \del{B \cd (1 + C\ln(1 + \ln(\frac{T}{u})))} \cd \EE\sbr{ \sum_{k=1}^{\infty}  \sum_{q=1}^\infty \one\cbr{ \bar X_k \in \lt\lparen\eps+\fr{q-1}{2}\eps,~ \eps + \fr{q}{2} \eps \rt\rbrack  } \cd \fr{\exp\del{\fr{1}{2\sigma^2} k \cd \del{\fr{q\eps}{2}}^2}}{1+\fr{D}{2\sigma^2}\cd k(\fr{q\eps}{2})^2} }~.
\end{align*} is upper bounded by
\begin{align*}
    & \del{B \cd (1 + C\ln(1 + \ln(\frac{T}{u})))} \cd  \del{\frac{40}{D} + 8} \fr{\sig^2}{\eps^2}
    \\& <\del{B \cd (1 + C\ln(1 + \ln(\frac{T}{\fr{2\sigma^2(1+c)^2\ln(T\Delta_a^2/(2\sigma^2))}{\Delta_a^2}})))} \cd  \del{\frac{40}{D} + 8} \fr{\sig^2}{\eps^2}
     \\&<O\del{B\cd (1+C(\ln(1+\ln(\fr{T\Delta_a^2}{\sigma^2}))))}\cd \del{\frac{40}{D} + 8} \fr{\sig^2}{\eps^2}
\end{align*}
where the last inequality is by the safe assumption of $\ln(T\Delta_a^2/(2\sig^2)) \ge 1$ since otherwise we have a trivial bound on $u$ as follows: $\frac{2\sig^2}{\Delta_a^2} e > T \ge u$.
Altogether, we choose $\blue{\eps} := \fr{c\Delta_a}{2(1+c)} $ and show that
\begin{align*}
  \EE[N_{t,a}]
  &\le \fr{2\sigma^2(1+c)^2 \ln(\fr{T\Delta_a^2}{2\sigma^2}(1+\ln(1+\fr{T\Delta_a^2}{2\sigma^2})))}{\Delta_a^2}
  + \frac{(1+2D)\sigma^2}{B\Delta_a^2} + \fr{2\sigma^2}{\del{\fr{c\Delta_a}{1+c} - \eps}^2} 
 \\ &+O\del{B\cd (1+C(\ln(1+\ln(\fr{T\Delta_a^2}{\sigma^2}))))}\cd \del{\frac{40}{D} + 8} \fr{\sig^2}{\eps^2}
   + O(1)
  \\&\le \fr{2\sigma^2(1+c)^2 \ln(\fr{T\Delta_a^2}{2\sigma^2}(1+\ln(1+\fr{T\Delta_a^2}{2\sigma^2})))}{\Delta_a^2} + O\del{1 + \del{\fr{BC}{D} + \fr{1}{B} + \fr{(1+c)^2}{c^2}}\cd \fr{\sigma^2\ln(1+\ln(\fr{T\Delta_a^2}{\sigma^2}))}{\Delta_a^2} }~.
\end{align*}

In terms of the upper bound of the regret, for every $T\ge1$ and $c>0$, 
\begin{align*}
    \EE \Reg_T &= \sum_{a\in[K]:\Delta_a>0} \Delta_a \cd \EE[N_{t,a}]
    \\&\le \sum_{a\in[K]:\Delta_a>0} \fr{2\sigma^2(1+c)^2 \ln(\fr{T\Delta_a^2}{\sigma^2}(1+\ln(1+\fr{T\Delta_a^2}{\sigma^2})))}{\Delta_a} + O\del{\Delta_a + \del{\fr{BC}{D} + \fr{1}{B} + \fr{(1+c)^2}{c^2}} \fr{\sigma^2\ln(1+\ln(\fr{T\Delta_a^2}{\sigma^2}))}{\Delta_a}}
     \\&\le \sum_{a:\Delta_a>0} \fr{2\sigma^2(1+c)^2 \ln(\fr{T\Delta_a^2}{2\sigma^2}(1+\ln(1+\fr{T\Delta_a^2}{2\sigma^2})))}{\Delta_a} + O\del{\Delta_a + \del{\fr{BC}{D} + \fr{1}{B} + \fr{(1+c)^2}{c^2}} \fr{\sigma^2\ln(1+\ln(\fr{T\Delta_a^2}{\sigma^2}))}{\Delta_a}}
\end{align*}
 where we omit the dependence on $C$ for brevity.
 
To verify the asymptotic optimality, let us choose  $c = \ln^{-1/4}(T)$ and take $B$, $C$, and $D$ as absolute constants.
Then,
\begin{align*}
    \lim_{T\to\infty} \fr{\EE[N_{t,a}]}{\ln(T)}
    &\le \lim_{T\to\infty} \fr{2\sigma^2(1+c)^2 \ln(\fr{T\Delta_a^2}{\sigma^2}(1+\ln(1+\fr{T\Delta_a^2}{\sigma^2})))}{\Delta_a^2 \ln(T)} +\lim_{T\to\infty} \fr{O\del{1 \vee \del{\fr{\sigma^2(1+c)^2\ln(1+\ln(\fr{T\Delta_a^2}{\sigma^2}))}{ c^2\Delta_a^2}}}}{\ln(T)}
    \\&=\lim_{T\to\infty} \fr{2\sigma^2(1+c)^2 \ln(\fr{T\Delta_a^2}{\sigma^2}(1+\ln(1+\fr{T\Delta_a^2}{\sigma^2})))}{\Delta_a^2 \ln(T)}
    \\&=\fr{2\sigma^2}{\Delta_a^2}~.
\end{align*}
This shows the asymptotic optimality of \MSP.

\end{proof}

\subsection{Proof of Corollary \ref{thm2_cor}}
\begin{proof}
Let us set $B$, $C$ and $D$ as absolute constants.
From the Proof of \MSP in Appendix \ref{Pfs_MSP}, in terms of the upper bound of the regret, 
\begin{align*}
   \EE \Reg_T &= \sum_{a\in[K]:\Delta_a>0} \Delta_a \cd \EE[N_{t,a}]
    \\ &= \sum_{a:\Delta_a<\Delta} \Delta_a \cd \EE[N_{t,a}]+
    \sum_{a:\Delta_a \ge \Delta} \Delta_a \cd \EE[N_{t,a}]
    \\ &< T\Delta + K \del{\fr{2\sigma^2(1+c)^2 \ln(\fr{T\Delta^2}{\sigma^2}(1+\ln(1+\fr{T\Delta^2}{\sigma^2})))}{\Delta} + O\del{\sigma^2\fr{(1+c)^2\ln(1+\ln(\fr{T\Delta_a^2}{\sigma^2}))}{c^2\Delta}}}+\sum_{a:\Delta_a\ge \Delta} O(\Delta_a) ~.
\end{align*}
By choosing $\Delta=\Theta(\sigma\sqrt{\fr{K\ln(K)}{T}})$, we have:
\begin{align*}
    \EE \Reg_T
    &=O(\sigma\sqrt{KT\ln(K)})+K
    \del{\fr{2\sigma^2(1+c)^2 \ln(\fr{T\Delta^2}{\sigma^2}(1+\ln(1+\fr{T\Delta^2}{\sigma^2})))}{\Delta} + 
    O \del{\sigma^2\fr{(1+c)^2\ln(1+\ln(\fr{T\Delta_a^2}{\sigma^2}))}{c^2\Delta}}}\\&\qquad + O(\sum_{a\in[K]} \Delta_a)
    \\&= O\del{\sigma\sqrt{KT\log(K)}} ~.
\end{align*}
Hence this algorithm achieves both $\sqrt{\ln(K)}$ minimax ratio and the asymptotically optimality.
\end{proof}



\end{document}